\documentclass[runningheads]{llncs}
\usepackage{graphicx}
\usepackage{xcolor}
\usepackage{color}
\usepackage{algpseudocode}
\usepackage{algorithm}
\usepackage{tikz}
\usepackage[pagebackref,breaklinks,colorlinks,bookmarks=false,urlcolor=black,citecolor={green!60!black}]{hyperref}
\usepackage{orcidlink}
\usepackage{comment}
\usepackage{dsfont}
\usepackage{amsmath,amssymb} % define this before the line numbering.
\usepackage{bm}
\usepackage{xspace}
\usepackage{enumitem}
\usepackage{caption}
\usepackage{subcaption}
\usepackage[accsupp]{axessibility}  % Improves PDF readability for those with disabilities.
\usepackage{url}

\usepackage{mathtools}
\renewcommand\vec[1]{\ensuremath\boldsymbol{#1}}
\renewcommand\cdots{...}
\newcommand{\tB}{\vec{\mathcal{B}}}

\newcommand{\tZ}{\vec{\mathcal{Z}}}
\newcommand{\tT}{\vec{\mathcal{T}}}

\newcommand{\tX}{\vec{\mathcal{X}}}
\newcommand{\tM}{\vec{\mathcal{M}}}

\newcommand{\vq}{\mathbf{q}}
\newcommand{\vk}{\mathbf{k}}
\newcommand{\vv}{\mathbf{v}}
\newcommand{\mbrp}[1]{\mathbb{R}_{+}^{#1}}
\newcommand{\mbr}[1]{\mathbb{R}^{#1}}

\newcommand{\tI}{\vec{\mathcal{I}}}

\newcommand{\tInb}{\mathcal{I}}

\newcommand{\idx}[1]{\mathcal{I}_{#1}}

\newcommand{\vphi}{\boldsymbol{\phi}}
\newcommand{\vpsi}{\boldsymbol{\psi}}
\newcommand{\vPsi}{\boldsymbol{\Psi}}

\newcommand{\bigoh}{\mathcal{O}}
\newcommand{\mPsi}{\vec{\Psi}}

\DeclareMathOperator*{\argmin}{arg\,min}

\DeclareMathOperator*{\trace}{Tr}
\DeclareMathOperator*{\kronstack}{\uparrow\!\otimes}

\DeclareMathOperator*{\diag}{Diag}

\newcommand{\mLambda}{\bm{\lambda}}
\newcommand{\mU}{\bm{U}}
\newcommand{\mV}{\bm{V}}
\newcommand{\mQ}{\bm{Q}}
\newcommand{\mK}{\bm{K}}
\newcommand{\mX}{\bm{X}}

\newcommand{\vw}{\boldsymbol{w}}

\makeatletter
\DeclareRobustCommand\bmvaOneDot{\futurelet\@let@token\bmv@onedotaux}
\def\bmv@onedotaux{\ifx\@let@token.\else.\null\fi\xspace}
\makeatother
\def\eg{\emph{e.g}\bmvaOneDot}

\def\etal{\emph{et al}\bmvaOneDot}

\def\ie{\emph{ie}\bmvaOneDot}

\def\wrt{w.r.t\bmvaOneDot}
\def\aka{a.k.a\bmvaOneDot}

\def\vs{\emph{vs}\bmvaOneDot}

\newcommand{\tG}{\boldsymbol{\mathcal{G}}}

\newcommand{\mygthreee}[2]{\boldsymbol{\mathcal{G}}_{{\text{#1}}}\!\left(\!#2\!\right)}

\newcommand{\mygthreeehat}[2]{\boldsymbol{\widehat{\mathcal{G}}_{{\text{#1}}}}\!\left(\!#2\!\right)}

\newcommand{\vPhi}{\boldsymbol{\Phi}}

\newcommand{\mIdent}{\boldsymbol{\mathds{I}}}
\newcommand{\sIdent}{\mathds{I}}

\newcommand{\mPhi}{\boldsymbol{\Phi}}

\newcommand{\mM}{\boldsymbol{M}}
\newcommand{\mF}{\boldsymbol{F}}

\newcommand{\mW}{\boldsymbol{W}}

\newcommand{\vmu}{\boldsymbol{\mu}}

\newcommand{\stkout}[1]{{\ifmmode\text{\sout{\ensuremath{#1}}}\else\sout{#1}\fi}}

\usepackage{multirow}
\usepackage{booktabs}
\usepackage{stmaryrd}
\usepackage{marginnote}

\usepackage{tcolorbox}
\usepackage{colortbl}
\usepackage{multibib}
\newcites{latex}{References}

\usepackage{filecontents}

\makeatletter
\newcommand\fs@nobottomruled{\def\@fs@cfont{\bfseries}\let\@fs@capt\floatc@ruled
  \def\@fs@pre{}% \hrule height.8pt depth0pt \kern2pt
  \def\@fs@post{}% Formerly \def\@fs@post{\kern2pt\hrule\relax}%
  \def\@fs@mid{\kern2pt\hrule\kern2pt}%
  \let\@fs@iftopcapt\iftrue}
\makeatother

\DeclareSymbolFont{extraup}{U}{zavm}{m}{n}
\DeclareMathSymbol{\varheart}{\mathalpha}{extraup}{86}
\DeclareMathSymbol{\vardiamond}{\mathalpha}{extraup}{87}

\DeclarePairedDelimiterX{\infdivx}[2]{(}{)}{%
  #1\delimsize\;#2%
}

\makeatletter
\renewcommand{\paragraph}{%
  \@startsection{paragraph}{4}%
  {\z@}{0.75ex \@plus 1ex \@minus .2ex}{-0.3em}%
  {\normalfont\normalsize\bfseries}%
}
\makeatother

\begin{document}
% \linenumbers
\pagestyle{headings}
\mainmatter
\def\ECCVSubNumber{5009}  % Insert your submission number here

%\title{TENET-Transformer is What You Need for Few-Shot Object Detection} % Replace with your title
\title{Time-rEversed diffusioN tEnsor Transformer:\\A new TENET of Few-Shot Object Detection}

\begin{comment}
\titlerunning{ECCV-22 submission ID \ECCVSubNumber} 
\authorrunning{ECCV-22 submission ID \ECCVSubNumber} 
\author{Anonymous ECCV submission}
\institute{Paper ID \ECCVSubNumber}
\end{comment}
%******************

%\begin{comment}
\titlerunning{Time-rEversed diffusioN tEnsor Transformer (TENET)}
% an abbreviated paper title here\textsuperscript{\textasteriskcentered}$^{\!, \dagger}
%
\author{Shan Zhang$^{\star, \dagger}$\orcidlink{0000-0002-5531-3296} \and
Naila Murray$^{\clubsuit}$\orcidlink{0000-0001-7032-0403} \and
Lei Wang$^{\vardiamond}$\orcidlink{0000-0002-0961-0441} \and
Piotr Koniusz\thanks{SZ was mainly in charge of the pipeline/developing the transformer. PK (corresponding author) was mainly in charge of mathematical design of TENET \& TSO.\\
This work has been accepted at the 17\textsuperscript{th} European Conference on Computer Vision (ECCV'22).$\qquad\qquad$ Code: {\fontsize{8}{8}\selectfont\url{https://github.com/ZS123-lang/TENET}}.}$^{,\S,\dagger}$\orcidlink{0000-0002-6340-5289}}
\authorrunning{Zhang \etal}
\institute{$^{\dagger}$Australian National University \;
$^{\clubsuit}$Meta AI \\
   $^{\vardiamond}$University of Wollongong \;
   $^\S$Data61/CSIRO\\
   %\tt\small 
   $^{\dagger}$firstname.lastname@anu.edu.au, $^{\vardiamond}$leiw@uow.edu.au,  $^{\clubsuit}$murrayn@fb.com
}
%\end{comment}
%******************
\maketitle

\vspace{-0.4cm}
\begin{abstract}
In this paper, we tackle the challenging problem of Few-shot Object Detection. Existing FSOD pipelines (i) use average-pooled representations that result in information loss; and/or (ii) discard position information that can help detect object instances.
Consequently, such pipelines are sensitive to large intra-class appearance and geometric variations between support and query images.
To address these drawbacks, we propose a Time-rEversed diffusioN tEnsor Transformer (TENET), which i) forms high-order tensor representations that capture multi-way feature occurrences that are  highly discriminative, and ii) uses a transformer that dynamically extracts correlations between the query image and the entire support set, instead of a single average-pooled support embedding. 
We also propose a Transformer Relation Head (TRH), equipped with higher-order representations, which encodes correlations between query regions and the entire support set, while being sensitive to the positional variability of object instances.
Our  model achieves state-of-the-art results on PASCAL VOC, FSOD, and  COCO.
\keywords{few-shot object detection; transformer; multiple order pooling; high order pooling; heat diffusion process;}
\end{abstract}

\section{Introduction}
\label{sec:intro}

%Object detection is a classic computer vision problem  usually  addressed by supervised models. 
Object detectors based on deep learning, usually  addressed by supervised models,  achieve impressive performance \cite{re30,re31,re32,7,9,10} but they  rely on a large number of images with human-annotated  class labels/object bounding boxes. Moreover, object detectors cannot be easily extended to new class concepts not seen during training. %es for which samples were not provided  during  training. 
Such a restriction limits supervised object detectors to predefined scenarios. In contrast, humans  excel at rapidly adapting to new scenarios by {\em ``storing knowledge gained while solving one problem and applying it to a different but related problem''}~\cite{west_ml_transfer_def}, also called as ``{\em transfer of practice}''~\cite{woodworth_particle}.

Few-shot Object Detection (FSOD) \cite{f9LSTD,f10ObjectDetection,f11,f12,fsod,accv,Zhang_2022_CVPR} methods mimic this ability, and  enable detection of test classes that are disjoint from training classes. They perform this adaptation using a few ``support'' images from  test classes. Successful FSOD models must (i) find promising candidate regions-of-Interest (RoIs) in query images; and (ii) accurately regress bounding box locations and predict RoI classes, under large intra-class geometric and photometric variations. 

To address the first requirement,  approaches \cite{f12,fsod,accv,Zhang_2022_CVPR} use the region proposal network \cite{re32}. For example, FSOD-ARPN \cite{fsod}, PNSD \cite{accv} and KFSOD \cite{Zhang_2022_CVPR} cross-correlate query feature maps with a class prototype formed from average-pooled (\ie, first-order) features, second-order pooled representations and kernel-pooled representations, respectively.
These methods use a single class prototype which limits their ability to leverage diverse information from different class samples.
%However, these methods still suffer from the uncertainty representations of support objects caused by a single class-wise prototype and insufficient statistics (first/second- order), in which the RPN is obstructed to attenuate irrelevant features (negative objects and non-objects) and generate precise set of candidate proposals belonging to the same category as support. 
Inspired by Transformers \cite{trans}, approach \cite{persampel} uses average pooling over support feature maps to generate a vector descriptor per  map. %, but does not aggregate these into a single class prototype. 
Attention mechanism is then used to modulate query image features using such descriptors.
%transforms features extracted from each support sample into query-position-aware vectors, instead of simply averaging such features to generate a single prototype per class.
%,  but it still applies first-order descriptors of varying sizes, pose, repeatable visual stimuli, \etc. 

The above methods rely on first- and second-order pooling, while so-called higher-order pooling is more discriminative  \cite{11,maxexp,hosvd}. 
Thus, we propose a non-trivial Time-rEversed diffusioN tEnsor Transformer (TENET). With TENET, higher-order tensors undergo a time-reversed heat diffusion to condense signal on super-diagonals of tensors, after which coefficients of these super-diagonals are passed to a Multi-Head
Attention transformer block.  
TENET performs second-, third- and fourth-order pooling. % via high-order tensors which %, as shown in our experiments, 
%lead to highly-discriminative representations. 
However, higher-order pooling suffers from several issues, \ie, (i) high computational complexity of computing tensors with three/more modes, (ii) non-robust tensor estimates due to the limited number of vectors being aggregated, and (iii) tensor burstiness\cite{11}.

To this end, we propose a Tensor Shrinkage Operator (TSO) which generalizes spectral power normalization (SPN) operators \cite{maxexp}, such as the Fast Spectral MaxExp operator (MaxExp(F)) \cite{maxexp}, to higher-order tensors. As such, it can be used to reduce tensor burstiness.
Moreover, by building on the linear algebra of the heat diffusion process (HDP) \cite{smola_graph} and recent generalisation of HDP to SPN operators \cite{maxexp}, we also argue that such operators  can reverse the diffusion of signal in autocorrelation or covariance matrices, and high-order tensors, instead of just reducing the burstiness. Using a parametrization which lets us control the reversal of diffusion, TSO  condenses signal captured by a tensor toward its super-diagonal, preserving information along it.
This super-diagonal serves as our final representation, reducing the feature size from $d^r$ to $d$,  making our representation  computationally tractable. 
Finally, shrinkage operators are known for their ability to estimate covariances well when only a small number of samples are available \cite{Ledoit110}.
%, with SPNs such as MaxExp and Gamma \cite{11,lin2017improved,maxexp} being among top performers in CV. 
To the best of our knowledge, we are the first to show that MaxExp(F) is a shrinkage operator, and to propose TSO for orders $r\!\geq\!2$.

To address the second requirement, FSOD-ARPN introduces a multi-relation head that captures global, local and patch relations between support and query objects, while PNSD passes second-order autocorrelation matrices to a similarity network. However, FSOD-ARPN and PNSD do not model  spatial relations  \cite{vit}. 
The QSAM \cite{persampel} uses attention to highlight the query RoI vectors that are similar to the set of support vectors (obtained using only first-order spatial average pooling).
%However, the lack of representative information and ignoring spatial variability would cause difficulty in measuring object-wise correlations for query objects detection. 
Thus, we introduce a Transformer Relation Head (TRH) to improve modeling of spatial relations.
TRH computes self-attention between spatially-aware features and global spatially invariant first-, second- and higher-order TENET representations of support and/or query RoI features. %The transformed features better preserve discriminative local information.
The second attention mechanism of TRH performs cross-attention between %%class prototypes represented  by a set of support sample embeddings
$Z$ support embeddings (for $Z$-shot if $Z\!\geq\!2$), and a set of global representations of query RoIs. This attention encodes similarities between query RoIs and support samples. 
%To generate the global representations used in the above attention modules, we use both average-pooled representations and, for better information-preservation, TENET representations.

% In the first, the representations are TRH uses TENET embeddings to represent test image and support image RoIs and uses themto is equipped with TENET, to extract , Dual-Awareness Detector (DAD). DAD equipped with TENET where both position and prototype correspondences are considered. Concretely, we first integrate the first-order and TENET-based embeddings via element-wise addition in a prototype-awareness detector. And then analogy to TENET applied to ARPN which extracts the spatial correlation between the query image and TENET-based prototypes of query, we build the individual position-aware features for support and query images, where the TENET-based and first-order vectors are regarded as super tokens concatenated with flattened pixel-wise features. And then cross-corresponds are assessed by various correlation measure operations to estimate diverse patterns along the channel (first/multiple- order statistics) and spatial dimension (pixel-wise features). This novel detector with TENET not only captures more informative statistics than other detector designs but is sensitive to the correlation of support-query pairs in prototype- and position- aware, making our model highly flexible and adaptive to varying object shapes and sizes. 

Our FSOD pipeline  contains TENET Region Proposal Network (TENET RPN) and the TRH, both equipped with discriminative TENET representations, improving   generation of RoI proposals and modeling of query-support relations.
% In our FSOD pipeline, we equip the transformer block with second-, third- and fourth-order super-diagonal representations of tensors from feature maps to quantify the contents of support and query regions for the use by the  Attention Region Proposal Network (ARPN) and the Dual-Awareness Detector (DAD). As TSO by design decorrelates high-order patterns by reversing the diffusion of signal, it leads to an improved attention realized by ARPN and an improved matching in DAD, with the latter capturing the prototype-awareness and position-awareness. 
%Figure \ref{fig:pipeline}  illustrates our pipeline which, on the popular PASCAL VOC 2007 benchmark \cite{voc}, obtains \textbf{6.1}\% and \textbf{4.4}\% absolute improvement (mAP) for novel and base classes over PNSD \cite{accv}, the second best performer (5-shot learning protocol). 
%

\vspace{0.2cm}
Below are our contributions:
\renewcommand{\labelenumi}{\roman{enumi}.}
\vspace{-0.1cm}
%\hspace{-1.0cm}
\begin{enumerate}[leftmargin=0.6cm]
\item We propose a Time-rEversed diffusiON tEnsor Transformer, called TENET, which captures high-order patterns (multi-way feature cooccurrences) and decorrelates them/reduces tensor burstiness. To this end, we generalize the MaxExp(F) operator \cite{maxexp} for autocorrelation/covariance matrices to higher-order tensors by introducing the so-called Tensor Shrinkage Operator (TSO).%\vspace{-0.2cm}
%\item Considering that support regions should be summarized by high-order patterns that are free of counts (\eg, small/large area covered by the fur, captured by multi-way occurrences should not decrease one's belief that an animal is observed), and also should be fully aggregated with query features by use of the entire support set, we build the TENET that transforms support super-diagonals of tensors into query vectors. %\vspace{-0.2cm}
\item We propose a Transformer Relation Head (TRH) that is sensitive both to the variability between the $Z$ support samples provided in a $Z$-shot scenario, and to positional variability between support and query objects.%prototype?, processing the first-order and TENET-based embeddings and pixel-wise maps.%\vspace{-0.1cm}
\item In \S \ref{sec:tempt}, we demonstrate that TSO emerges from the MLE-style minimization over the Kullback-Leibler (KL) divergence between the input and output spectrum, with the latter being regularized by the Tsallis entropy \cite{tsallis}. Thus, we show that TSO meets the definition of shrinkage estimator whose target is the identity matrix (tensor).%\vspace{-0.2cm}
\end{enumerate}
Our proposed method outperforms the state of the art on novel classes by 4.0\%, 4.7\% and 6.1\% mAP on PASCAL VOC 2007, FSOD, and COCO respectively.

\section{Related Works}
\label{sec:related}

Below, we review  FSOD models and vision transformers, followed by a  discussion on feature grouping, tensor descriptors and spectral power normalization.

\paragraph{Few-shot Object Detection.}
A Low-Shot Transfer Detector (LSTD) \cite{f9LSTD} leverages rich source domain to construct a target domain detector with few training samples but needs to be fine-tuned to novel categories. %, whereas meta-learning does not require fine-tuning. 
%Approach \cite{f10ObjectDetection} is a single-stage detector combined with a meta-model that reweights the importance of features from the base model to adjust contributions of base features to new classes. Similarly,
Meta-learning-based approach \cite{f12} reweights RoI features in the detection head without fine-tuning. Similarly, MPSR \cite{multi} deals with scale invariance by ensuring the detector is trained over multiple scales of positive samples. NP-RepMet \cite{neg} introduces a negative- and positive-representative learning framework via triplet losses that bootstrap the classifier. FSOD-ARPN \cite{fsod} is a general FSOD network equipped with a channel-wise attention mechanism and multi-relation detector that scores pair-wise object similarity in both the RPN and the detection head, inspired by Faster R-CNN. PNSD \cite{accv}, inspired by FSOD-ARPN \cite{fsod}, uses contraction of second-order autocorrelation matrix against query feature maps to produce attention maps. % while second-order matrices are passed to the detection head. 
Single-prototype (per class) methods  % when performing aggregation, 
suffer information loss. Per-sample Prototype FSOD \cite{persampel} uses the entire support set to form prototypes of a class but it ignores spatial information within regions. Thus, we employ TENET RPN and TRH %, that is, Time-rEversed diffusioN tEnsor Transformer, 
to capture spatial and high-order patterns, and extract correlations between the query image and the $Z$-shot support samples for a class. %Our pipeline is designed to transform  these multiple-order representations to position-aware features in ARPN and TRH.

%\vspace{0.1cm}
%\noindent\textbf{Transformers in Vision.} 
\paragraph{Transformers in Vision.}
Transformers, popular in natural language processing \cite{trans}, have also become  popular in  computer vision. Pioneering works such as ViT \cite{vit} show that transformers can  achieve the state of the art in image recognition. DETR  \cite{detr} is an end-to-end object detection framework with a transformer encoder-decoder used on top of backbone. Its deformable variant \cite{defordetr} improves the performance/training efficiency. SOFT \cite{rbf}, the SoftMax-free transformer  approximates the self-attention kernel  by replacing the SoftMax function with Radial Basis Function (RBF), achieving linear complexity. In contrast, our TENET is concerned with reversing the diffusion of signal in high-order tensors via the shrinkage operation, with the goal of modeling spatially invariant high-order statistics of regions. Our attention unit, so-called Spatial-HOP in TRH, also uses RBF to capture correlations of between spatial and high-order descriptors. %, that is shift-invariant in physical location, orientation, viewpoint, in the multi-head attention unit \cite{rbf}. %and have parameters to control the model complexity \eg, an RBF kernel with a small (resp. large) radius captures a complex (resp. simple) decision boundary, we follow the method \cite{rbf} in the multi-head attention unit.

%\vspace{0.1cm}
%\noindent\textbf{{Multi-path and Groups of Feature Maps}}.
\paragraph{Multi-path and Groups of Feature Maps.}
GoogleNet \cite{g52} has shown that multi-path representations (several network branches) lead to classification improvements. ResNeXt \cite{g61} adopts group convolution \cite{g34} in the ResNet bottleneck block. %, which integrates the multi-path structure into a unified framework. 
SK-Net \cite{g38}, based on SE-Net \cite{g29}, uses feature map attention across two network branches. However,  these approaches do not model feature statistics. %Somewhat closer to our idea is bilinear pooling \cite{37}, which correlates two groups of feature channels from two regions, whereas
ReDRO \cite{saimunur_redro} samples groups of features to apply the matrix square root over submatrices to improve the computational speed. In contrast,  TENET forms fixed groups of features to form second-, third- and fourth-order tensors (simply using groups of features to form second-order matrices is not effective).

%\vspace{0.1cm}
%\noindent\textbf{Second-order Pooling (SOP).} 
\paragraph{Second-order Pooling (SOP).} 
Region Covariance Descriptors for texture \cite{36} and object recognition  \cite{11}   use SOP. %Approach \cite{37} uses two separate feature locations followed by an outer product. 
Approach  \cite{14} uses spectral pooling for fine-grained image classification, whereas SoSN \cite{sosn} and its variants \cite{zhang2020few,Zhang_2021_CVPR} leverage SOP and element-wise Power Normalization (PN) \cite{11} for end-to-end few-shot learning. In contrast, we develop a multi-object few-shot detector. Similarly to SoSN, PNSD \cite{accv} uses SOP with PN as representations  which are passed to the detection head. 
So-HoT \cite{me_domain} that uses high-order tensors for domain adaptation is also somewhat related to TENET but So-HoT uses multiple polynomial kernel matrices, whereas we apply TSO to achieve decorrelation and shrinkage. TENET without TSO reduces to polynomial feature maps and performs poorly.

%\vspace{0.1cm}
%\noindent\textbf{Power Normalization (PN).} 
\paragraph{Power Normalization (PN).} 
Burstiness is ``the property that a given visual element appears more times in an image than a statistically independent model would predict''. PN \cite{s17} limits the  burstiness of first- and second-order statistics due to the binomial PMF-based feature detection factoring out feature counts \cite{s17,11,14}. Element-wise MaxExp pooling  \cite{s17} gives likelihood of ``at least one particular visual word being present in an image'', whereas SigmE pooling \cite{14} is its practical approximation. Noteworthy are   recent Fast Spectral MaxExp operator, MaxExp(F) \cite{maxexp}, which   reverses the heat diffusion on the underlying loopy graph of second-order matrix to some desired past state \cite{smola_graph}, and Tensor Power-Euclidean (TPE) metric \cite{hosvd}. TPE alas uses the Higher Order Singular Value Decomposition \cite{lathauwer_hosvd}, which makes TPE intractable for millions of region proposals per dataset. Thus, we develop TENET and TSO, which reverses diffusion on high-order tensors by shrinking them towards the super-diagonal of tensor.

\section{Background}
\label{sec:pre}

Below, we detail notations %/tensor algebra pre-requisites,
and show how to calculate  multiple higher-order statistics and Power Normalization, followed by revisiting the transformer block.

%\vspace{0.1cm}
%\noindent\textbf{{Notations.}}
\paragraph{Notations.}
Let $\mathbf{x}\!\in\!\mbr{d}$ be a $d$-dimensional feature vector. $\idx{N}$ stands for the index set $\{1,2,\cdots,N\}$. %Moreover, for a matrix $\mX$, we denote its outer product as $\mX\mX^T\!\!\!=\,\uparrow\!\otimes_2\mX$. 
We define a vector of all-ones as $\mathbf{1}\!=\!\left[1,\cdots,1\right]^T$. 
Let $\tX^{(r)} ={\kronstack}_r \mathbf{x}$ denote a tensor of order $r$ generated by the $r$-th order outer-product of $\mathbf{x}$, and $\tX^{(r)}\in\mbr{d\times d\cdots\times d}$. 
%Calligraphic mathcal fonts denote tensors (\eg, $\tX$), capitalized boldface symbols are matrices (\eg, $\mX$), lowercase boldface symbols  are vectors (\eg, $\vx$), and regular fonts denote scalars \eg, $X_{i,j}$, $x_{i}$, $N$ ($X_{i,j}$ is the $(i,j)$-th coefficient of $\mX$). 
Typically, capitalised boldface symbols such as $\boldsymbol{\Phi}$ denote matrices, lowercase boldface symbols such as $\boldsymbol{\phi}$ denote vectors and regular case such as $\Phi_{i,j}$, $\phi_{i}$, $n$ or $Z$ denote scalars, \eg, $\Phi_{i,j}$ is the $(i,j)$-th coefficient of $\boldsymbol{\phi}$. %More information about tensors/popular notations is provided in  papers \cite{lathauwer_hosvd,kolda_tensorrew,11,hosvd}. 

%\vspace{0.1cm}
%\noindent\textbf
\paragraph{High-order Tensor Descriptors (HoTD).} 
%are formalized below.
 Below we  formalize  HoTD \cite{me_domain}. 
%HoTD do not include our TSO at this stage but they are a starting point in the design of TENET unit. 
\begin{proposition}
Let $\mPhi\equiv[\vphi_1,\cdots,\vphi_N]\in\mbr{d\times N}$
and
$\mPhi'\!\equiv[\vphi'_1,\cdots,\vphi'_M]\in\mbr{d\times M}$
be feature vectors extracted from some two image regions. 
%Let $\Pi$ and $\Pi^*$, and $N\!=\!|\tNnb|$ and $N^*\!\!=\!|\tNnb^*|$ be the numbers of data vectors \eg, obtained from the last convolutional feature map of CNN for instances $\Pi$ and $\Pi^*$ and
Let $\vw\in\mbrp{N}$, $\vw'\in\mbrp{M}$ be some non-negative weights and $\vmu,\vmu'\!\in\!\mbr{d}$ be the mean vectors of $\mPhi$
and $\mPhi'$, respectively.
A linearization of the sum of polynomial kernels of degree $r$,
\begin{align}
&\!\!\!\big<\tM^{(r)}(\mPhi; \vw,\vmu),\,\tM^{(r)}(\mPhi'\!; \vw'\!\!,\vmu'\!)\big> = %\nonumber\\
%&\!
\frac{1}{NM}\sum\limits_{n=1}^N\sum\limits_{m=1}^M w_n^rw'^{r}_m
\left<\vphi_n\!-\!\vmu, \vphi'_{m}\!-\!\vmu'\right>^r\!\!,
\end{align}
yields the tensor feature map
\begin{equation}
\!\!\tM^{(r)}(\mPhi; \vw,\vmu)=\frac{1}{N}\sum\limits_{n=1}^N w_n^r\,{\kronstack}_r\,(\vphi_n\!-\!\vmu)\in
\mbr{d\times d\cdots\times d}.\label{eq:hok1}
\end{equation}
\end{proposition}
In our paper, we set $\vw\!=\!\vw'\!=\!1$ and $\vmu\!=\!\vmu'\!=\!0$, whereas orders $r\!=\!2,3,4$.

%Specifically, we formulate the second/third/fourth-order feature map as: $\!\!\tM\!\times_2(\mPhi)=\frac{1}{N}\sum\limits_{n=1}^N \vphi_n\vphi_n^T\in
%\mbr{d\times d};
%\tM\!\times_3(\mPhi)=\frac{1}{N}\sum\limits_{n=1}^N  (\vphi_n\vphi_n^T)\vphi_n^T\in
%\mbr{d\times d\times d};
%\tM\!\times_4(\mPhi)=\frac{1}{N}\sum\limits_{n=1}^N (\vphi_n\vphi_n^T)(\vphi_n\vphi_n^T)\in
%\mbr{d\times d \times d\times d}$

%\vspace{0.1cm}
%\noindent\textbf
\paragraph{(Eigenvalue) Power Normalization ((E)PN).} For second-order matrices, MaxExp(F), a state-of-the-art  EPN \cite{maxexp}, is defined as
%
%\vspace{-0.1cm}
\begin{equation}
g(\lambda; \eta)\!=\!1-(1-\lambda)^\eta
\label{eq:maxexp}
\end{equation}
on the $\ell_1$-norm normalized spectrum from SVD ($\lambda_i\!:=\!\lambda_i/(\varepsilon\!+\!\sum_{i'} \lambda_{i'})$), and on %the space of 
symmetric positive semi-definite matrices as
%
%\vspace{-0.3cm}
\begin{equation}
\mygthreeehat{MaxExp}{\,\mM;\eta\,}\!=\!\sIdent\!-\!\left(\sIdent\!-\!\mM\right)^\eta,
\label{eq:maxexpf}
\end{equation}
where $\mM$ is a trace-normalized matrix, \ie, $\mM\!:=\!\mM/(\varepsilon\!+\!\trace(\mM))$ with $\varepsilon\!\approx\!1e\!-\!6$, and $\trace (\cdot)$ denotes the trace. The time-reversed heat diffusion process is adjusted by integer $\eta\!\geq\!1$. The larger the value of $\eta$ is, the more prominent the time reversal is. $\widehat{\mathcal{G}}_{\text{MaxExp}}$ is followed by the element-wise PN, called SigmE \cite{maxexp}:
%
%\vspace{-0.2cm}
\begin{align}
& \!\!\!\!\!%\Psi\!=\!
\mathcal{G}_{\text{SigmE}}(\mathit{p}; \eta') =\!2/(1+e^{-\eta' \mathit{p}})-1,
\label{eq:sigme1}
\end{align}
%
%which performs detection of feature co-occurrences instead of counting.
where $p$ takes each output coefficient of Eq. \eqref{eq:maxexpf}, % or  an average over samples of a feature,  
$\eta'\!\!\geq\!1$ controls  detecting feature occurrence \vs feature counting trade-off.
%Specifically, one may use the parametrization $\eta(t)\!\approx\!0.5\sqrt{1\!+\!4/(t^2(e\!-\!1)^2)}\!-\!0.5$, in which case, $\mygthreee{MaxExp}{\,\mM;\eta\,}\!\approx\!\mygthreee{HDP}{\,\mM;\eta\,}\!=\!\expm(\!-t\mM^{\dagger})$ \cite{maxexp} for the loopy graph  Laplacian matrix $\mM^{\dagger}$ corresponding to the autocorrelation/covariance matrix $\mM$. However, parametrization $\eta(t)$ is non-essential so we do not use it. Finally, where stated, we use element-wise PN called SigmE \cite{maxexp}, given as

%\vspace{0.1cm}
%\noindent\textbf
\paragraph{Transformers.} An architecture based on blocks of attention and MLP layers forms a transformer network \cite{trans}. Each attention layer takes as input a set of query, key  and value  matrices, denoted $\mQ$, $\mK$ and $\mV$, respectively.
Let $\mQ\equiv[\vq_1,\cdots,\vq_N]\in\mbr{d\times N}$, $\mK\equiv[\vk_1,\cdots,\vk_N]\in\mbr{d\times N}$, and $\mV\equiv[\vv_1,\cdots,\vv_N]\in\mbr{d\times N}$, where $N$ is the number of input feature vectors, also called tokens, and $d$ is the channel dimension. A generic attention layer can then be formulated as:
%\vspace{-0.1cm}
\begin{equation}
\text{Attention}(\mQ,\mK,\mV)=\alpha\big(\gamma(\mQ,\mK)\big) \mV^T.
\label{eq:atten}
%\vspace{-0.1cm}
\end{equation}
%where [$Q;K;V$] = [$W_qF_q;W_kF_k;W_vF_v$], in which $F_q$ is the input query sequence, $F_k/F_v$ is the input key/value sequence, $W_q, W_k, W_v \in \mbr{d \times d} $ are learnable weights. 
Self-attention
is composed of $\alpha(\gamma(\cdot,\cdot))$, $\alpha(\cdot)$ is a non-linearity and  $\gamma(\cdot,\cdot)$ computes similarity. A popular choice is SoftMax with the scaled dot product \cite{trans}:
\begin{equation}
\alpha(\cdot)=\text{SoftMax}(\cdot) \quad\text{and}\quad \gamma(\mQ,\mK)=\frac{\mQ^T\mK}{\sqrt{d}}.
%\vspace{-0.1cm}
\end{equation}
Note the LayerNorm and residual connections are added at the end of each block.

To facilitate the design of linear self-attention, approach \cite{rbf} introduces a SoftMax-free transformer with the dot product replaced by the RBF kernel:
\begin{equation}
\alpha(\cdot)=\text{Exp}(\cdot) \quad\text{and}\quad  \gamma(\mQ,\mK)=-\frac{1}{2\sigma^2}\left[\left\lVert \vq_i-\vk_j\right\rVert_2^2\right]_{i,j\in\idx{{N}}},
\label{eq:rbf}
%\vspace{-0.1cm}
\end{equation}
where $[\cdot]$ stacks distances into a matrix, $\sigma^2$ is the kernel variance set by cross-validation, and $\vq_i$ and $\vk_j$ are $\ell_2$-norm normalized. 
%We normalize the Q and K elements along the channel mode: $\frac{Q_{:,i}}{\left\lVert Q_{:,i} \right\rVert_2}$,$\frac{k_{:,i}}{\left\lVert k_{:,i} \right\rVert_2}$, 
%Then $\gamma(Q,K) \propto \big(\frac{ QK^T-1}{\sigma^2}\big)$.

The Multi-Head Attention (MHA) layer uses $T$ attention units whose outputs are concatenated. Such an attention splits input matrices $\mQ$, $\mK$, and $\mV$ along their channel dimension $d$ into $T$ groups and computes attention on each group:
%\vspace{-0.01cm}
\begin{equation}
\text{MHA}(\mQ,\mK,\mV)=[\mathbf{A}_1,\dots,\mathbf{A}_T],
\label{eq:mha}
%\vspace{-0.05cm}
\end{equation}
where $[\cdot]$ performs concatenation along the channel mode, the $m$\textsuperscript{th}  head is  $\mathbf{A}_m \!\!=\!\!\text{Attention}(\mQ_m, \mK_m, \mV_m)$, and $\mQ_m\!\in\!\mbr{\frac{d}{T}\times N}$, $\mK_m\!\in\!\mbr{\frac{d}{T}\times N}$, and $\mV_m\!\in\!\mbr{\frac{d}{T}\times N}$.

\begin{figure}[t]
%\vspace{-0.3cm}
%trim={<left> <lower> <right> <upper>}
%\centering\includegraphics[ width=1.0\linewidth]{images/pipe-shan-eccv.pdf}%images/pipe-shan-eccv-detailed.pdf
\centering\includegraphics[ width=1.0\linewidth]{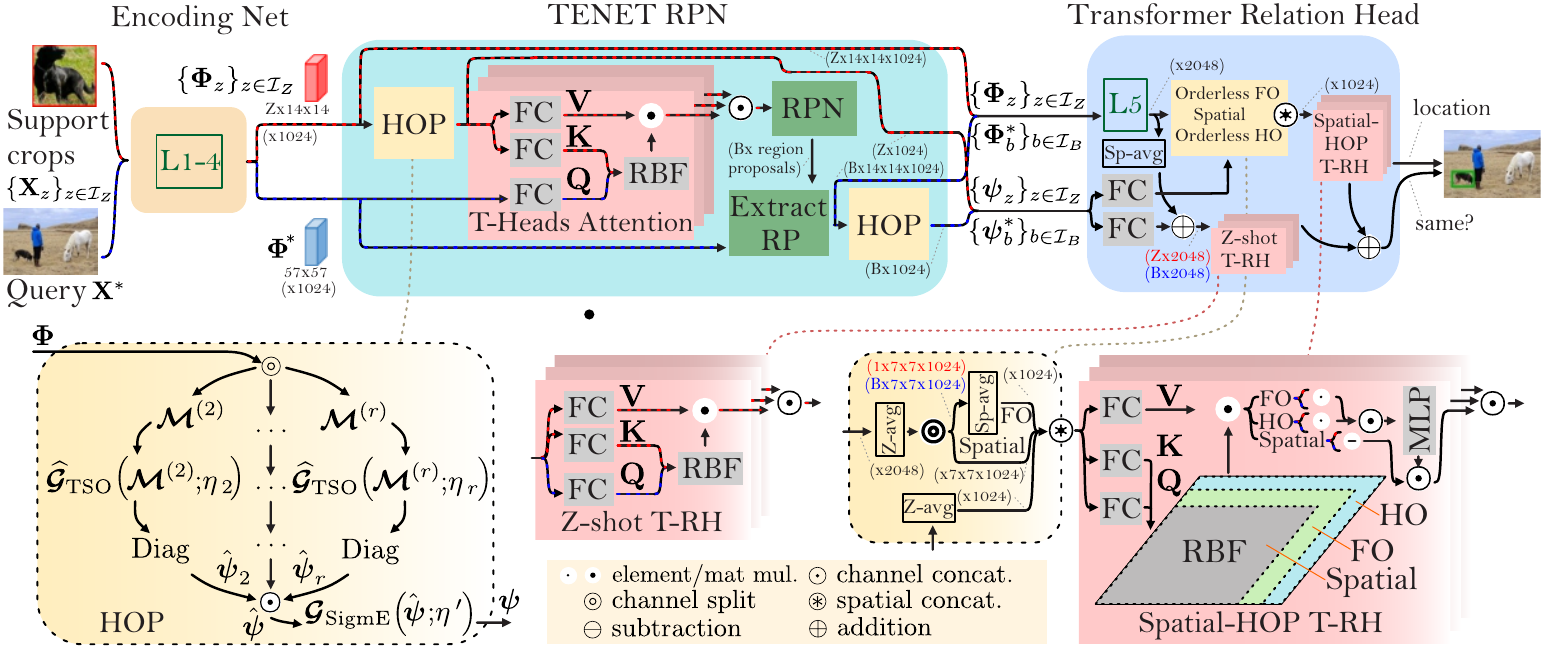}
%\vspace{-0.3cm}
\caption{Our pipeline. ({\em top}) We pass ground truth support crops $\{\tX_z\}_{z\in\idx{Z}}$ for $Z$-shot problem and a query image $\tX^*$ to the Encoding Network (EN). The resulting convolutional feature maps, $\{\boldsymbol{\Phi}_z\}_{z\in\idx{Z}}$ for support and $\boldsymbol{\Phi}^*$ for query, are passed to the TENET-RPN module to produce a set of $B$ descriptors $\{\boldsymbol{\Phi}_b^*\}_{b\in\idx{B}}$ for $B$ Region Proposals (RP \aka RoIs) of query image, and high-order pooled (HOP) representations for both support crops $\{\boldsymbol{\psi}_z\}_{z\in\idx{Z}}$ and query image RoIs $\{\boldsymbol{\psi}_b^*\}_{b\in\idx{B}}$.
TENET contains HOP units which compute high-order tensor descriptors and then apply a novel tensor shrinkage operator to them, yielding   spatially orderless HOP representations. Sets of features $\{\boldsymbol{\psi}_z\}_{z\in\idx{Z}}$ and $\{\boldsymbol{\psi}_b^*\}_{b\in\idx{B}}$ are then passed to the Transformer Relation Head (TRH), along with the convolutional features $\{\boldsymbol{\Phi}_z\}_{z\in\idx{Z}}$ and $\{\boldsymbol{\Phi}^*_b\}_{b\in\idx{B}}$. 
The TRH consists of $Z$-shot and Spatial-HOP transformer heads for measuring similarities across support regions and query proposals, and refining the localization of target objects, respectively. %The HOP descriptors of support images and query proposals are $\boldsymbol{\Psi}$ ($Z$ vector) and $\boldsymbol{\Psi}_b^*$ ($b\in\tB$, a set of query proposals). %$\oplus$ denotes element-wise addition and $\odot$ means matrix multiplication.
({\em bottom}) Details of HOP, ``Orderless FO, Spatial, Orderless HO'', Z-shot T-RH and Spatial-HOP T-RH blocks. %Note the attention Spatial-HOP block combines spatial, first- and high-order entries.  
See \S \ref{pipe-det} for details.
} 
\label{fig:embeddings}
%\vspace{-0.3cm}
\end{figure}

%\section{Conclusions}
\section{Proposed Approach}
\label{sec:approach}

Given  a set of $Z$ support crops $\{\mX_z\}_{z\in\idx{Z}}$ and a query image $\mX^*$ per episode, our approach learns a matching function between representations of query RoIs and and support crops. Fig.~\ref{fig:embeddings} shows our pipeline comprised of three main modules:
\begin{enumerate}
    \item \textbf{Encoding Network} (EN)  extracts feature map $\vPhi \in \mbr{d\times N}$ per image (of $N\!=\!W\!\times\!H$ spatial size) from  query and support images via ResNet-50. %, where $N\!=\!HW$. % and $H$ and $W$ are the spatial dimensions of the feature map.
    \item \textbf{TENET RPN} extracts RoIs from the query image and computes embeddings for  the query RoIs and the support crops. Improved attention maps are obtained by T-Heads Attention (THA) that operates on TENET descriptors. 
    \item \textbf{Transformer Relation Head} captures relations between query and support features using self- and cross-attention mechanisms. This head produces  representations %, one per RoI, that are 
    for the classifier and bounding-box refinement regression loss.
\end{enumerate}
 
Next, we describe TENET RPN, followed by the Transformer Relation Head.

\subsection{Extracting representations for support and query RoIs}
\label{sec:tempt}

%\centering\includegraphics[ width=1.0\linewidth]{images/pipe-shan-eccv-detailed.pdf}
%\vspace{-0.3cm}
%\caption{Detailed illustration of our pipeline. We follow the architecture from Fig. \ref{fig:embeddings} and expand blocks such as HOP, ``Orderless FO, Spatial, Orderless HO'', Z-shot T-RH and Spatial-HOP T-RH. See the text for detailed descriptions.} 

Fig.~\ref{fig:embeddings} ({\em top}) shows our TENET RPN that produces embeddings and query RoIs. Central to this module is our HOP unit that produces  Higher-order Tensor Descriptors (HoTDs). HOP splits features  along the channel mode into multiple groups of feature maps, from which second-, third- and fourth- order tensors are aggregated over desired regions. HoTDs use a generalization of the MaxExp(F) to higher-order tensors, called the Tensor Shrinkage Operator (TSO).

%\vspace{0.1cm}
%\noindent\textbf
\paragraph{Tensor Shrinkage Operator.}  Ledoit and Wolf \cite{Ledoit110} define autocorrelation/ covariance matrix estimation as a trade-off between the sample matrix $\mM$ and a highly structured operator $\mF$, using the linear combination $(1\!-\!\delta)\mM\!+\!\delta\mF$.
For symmetric positive semi-definite matrices and tensors, one can devise a convex shrinkage operator by minimizing some divergence $d(\mLambda, \mLambda')$ between the source and target spectra, where $\mLambda'$ is  regularized by $\Omega(\mLambda')$ with weight $\delta\!\geq\!0$:
\begin{align}
& \mLambda^*\!=\!\argmin_{\mLambda'\geq 0} d(\mLambda, \mLambda') + \delta\Omega(\mLambda').
\label{eq:shr}
\end{align}
%
%Let $\boldsymbol{\lambda}$ be the $\ell_1$-norm normalized spectrum\, as in Eq. \eqref{eq:maxexp}. Then $g(\lambda; \eta)\!=\!1-(1-\lambda)^\eta$ is in fact a shrinkage operator as it is a solution to the problem in Eq. \eqref{eq:shr} if  $\delta\!=\!1$, and the Kullback-Leibler  divergence and the Tsallis entropy are substituted for the divergence term $d$ and the regularization term $\Omega$, respectively. Please refer to  \S \ref{sec:tso} of supplementary material for the proof, where we also discuss the highly structured operator $\mF$, \ie, the target of the shrinkage operator, is the identity matrix.
Below we derive the TSO as a generalization of MaxExp(F)\footnote{For $r\!=\!2$, Eq. \eqref{eq:tso_1} yields $\text{Diag}(\mIdent\!-\!(\mIdent\!-\!\mM)^{\eta_2})$. $\text{Diag}(\text{Sqrtm}(\mM))$ is its approximation.}.

\definecolor{beaublue}{rgb}{0.8, 0.85, 0.8}
\definecolor{blackish}{rgb}{0.2, 0.2, 0.2}
\vspace{0.1cm}
\begin{minipage}{0.92\linewidth}
\centering
\begin{tcolorbox}[grow to left by=0.01\linewidth, width=1\linewidth, colframe=blackish, colback=beaublue, boxsep=0mm, arc=2mm, left=2mm, right=2mm, top=2mm, bottom=2mm]

With $\mygthreeehat{TSO}{\,\cdot\,}$  we extract representations $\hat{\vpsi}_r$ from HoTDs $\tM^{(r)}$: 
%\vspace{-0.2cm}
\begin{align}
& \mygthreeehat{TSO}{\tM^{(r)};\eta\,}\!=\!\tI_r\!-\left(\tI_r-\tM^{(r)}\right)^{\eta}\!\!, \label{eq:tso1}\\
& \hat{\vpsi}_r\!=\!\diag\left(\mygthreeehat{TSO}{\tM^{(r)};\eta_r\,}\right), \label{eq:tso_1}\\
& \vpsi_r\!=\!\mygthreee{SigmE}{\,\hat{\vpsi}_r;\eta'_r\,},\label{eq:tso_2}
\end{align}
where %$\mygthreeehat{TSO}{\,\cdot\,}$ is the TSO,  
$\diag(\cdot)$ extracts the super-diagonal of tensor. The identity tensor $\tI_r$ of order $r$ is defined  such that  $\tInb_{1,\cdots,1}\!=\!\tInb_{2,\cdots,2}\cdots\!=\tInb_{d,\cdots,d}\!=\!1$ and $\tInb_{i_1,\cdots,i_r}\!=\!0$ if $i_j\!\neq\!i_k$ and  $j\!\neq\!k,\,j,k\!\in\!\idx{r}$. 
%\end{tcolorbox}
%\end{minipage}
%\vspace{0.1cm}

%Below we derive the TSO as a generalization of MaxExp(F)\footnote{For $r\!=\!2$, Eq. \eqref{eq:tso_1} yields $\text{Diag}(\mIdent\!-\!(\mIdent\!-\!\mM)^{\eta_2})$. $\text{Diag}(\text{Sqrtm}(\mM))$ is its approximation.}.

%Based on these observations, we can readily extend MaxExp(F) to general high-order tensors with TSO as follows:
%\begin{equation}
%\mygthreeehat{TSO}{\tM^{(r)};\eta\,}\!=\!\tI_r\!-\left(\tI_r-\tM^{(r)}\right)^{\eta},
%\label{eq:tso1}
%\end{equation}

\begin{theorem}
Let $\boldsymbol{\lambda}$ be the $\ell_1$-norm normalized spectrum, as in Eq. \eqref{eq:maxexp}. Then $g(\lambda; \eta)\!=\!1-(1-\lambda)^\eta$ is  a shrinkage operator as it is a solution to  Eq. \eqref{eq:shr} if  $\delta\!=\!\frac{\eta t^{1/\eta}}{s}(1-\frac{1}{\eta})$, and the Kullback-Leibler  divergence and the Tsallis entropy are substituted for the divergence term $d$ and the regularization term $\Omega$, respectively. In the above theorem, $\boldsymbol{\lambda}\in\mbr{d}_+$, $s\!=\!d\!-\!1$, $t\!=\!d\!-\!t'$ and $t'\!>\!0$ is the trace of $g(\boldsymbol{\lambda}; \eta)$.$\!\!\!$
%
%Let $\lambda$ be the $ell_1$-norm normalized spectrum, as in Eq. \eqref{eq:maxexp}. Then $g(\lambda; \eta)\!=\!1-(1-\lambda)^\eta$ is in fact a shrinkage operator, a solution to the problem in Eq. \eqref{eq:shr} with $\delta\!=\!1$, and the Kullback-Leibler  divergence and the Tsallis entropy substituting the distance/divergence\footnote{The KL divergence is not a distance, as it is not  a metric measure due to its non-symmetric nature.}  $d$ and the regularization term $\Omega$, respectively.
\label{th:shrink}
\end{theorem}
\end{tcolorbox}
\end{minipage}
\begin{proof}
Let $d(\mLambda, \mLambda')\!=\!
D_{\text{KL}}\big({\mLambda^\circ}\|{{\mLambda^\circ}'}\big)
%\infdiv{{\mLambda^\circ}}{{{\mLambda^\circ}'}}
\!=\!-\big(\!\sum\limits_{i\in\idx{d}}\!\lambda^\circ_i\log{\lambda_i^\circ}'\big)\!+\!\big(\!\sum\limits_{i\in\idx{d}}\lambda^\circ_i\log{\lambda_i^\circ}\big)$, where $\mLambda^\circ\!=\!\frac{1-\mLambda}{s}$ and ${\mLambda^\circ}'\!=\!\frac{1-\mLambda'}{t}$ are complements of $\mLambda$ and $\mLambda'$ normalized by $s\!=\!d\!-\!1$ and $t\!=\!d\!-\!t'$ to obtain normalized distributions to ensure valid use of the Kullback-Leibler divergence and the Tsallis entropy. Let  $\Omega(\mLambda';\alpha)\!=\!\frac{1}{\alpha-1}(1\!-\!\sum\limits_{i\in\idx{d}}{\lambda^\circ_i}'^\alpha)$. Define $f(\mLambda,\mLambda')\!=\!-\big(\!\sum\limits_{i\in\idx{d}}\!\lambda^\circ_i\log{\lambda_i^\circ}'\big)\!+\!\big(\!\sum\limits_{i\in\idx{d}}\!\lambda^\circ_i\log{\lambda_i^\circ}\big)\!+\!\frac{1}{\alpha-1}(1\!-\!\!\sum\limits_{i\in\idx{d}}\!{\lambda^\circ_i}'^\alpha)$ (following Eq. \eqref{eq:shr}) 
 which we minimize \wrt $\mLambda'$ by computing $\frac{\partial f}{\partial\lambda'_i}\!=\!0$, that is,
%
%\vspace{-0.3cm}
\begin{align}
&\nonumber\\[-20pt]
& \label{eq:der}\frac{\partial f}{\partial\lambda'_i}\!=\!0\Rightarrow\lambda'_i\!=\!1\!-\!t\Big(\frac{\eta}{\delta s}\Big)^\eta (1\!-\!1/\eta)^\eta(1\!-\!\lambda_i)^{\eta},\\[-20pt]\nonumber
\end{align}
%
%\vspace{-0.2cm}
where $\frac{1}{\alpha}\!=\!\eta$. Solving $t\big(\frac{\eta}{\delta s}\big)^\eta(1\!-\!\frac{1}{\eta})^\eta\!=\!1$ for $\delta$  completes the proof.
\end{proof}
%

%Moreover, from the above proof, it is easy to conclude that TSO is not a mere naive linear interpolation between $\lambda$ and the target 1.

\begin{theorem}
The highly structured operator $\mF$ (target of  shrinkage) equals $\mIdent$.
\label{th:ident}
%\vspace{-0.2cm}
\end{theorem}

\begin{proof}
Notice   $\lim\limits_{\eta\rightarrow\infty} 1\!-\!(1\!-\!\lambda_i)^\eta\!=\!1$ if $\mLambda\!\neq\!0$ is the $\ell_1$-norm normalized spectrum from SVD, \ie, $\mU\mLambda\mU^T\!\!=\!\mM\!\succcurlyeq\!0$ with $\lambda_i\!:=\!\lambda_i/(\sum_{i'} \lambda_{i'}\!+\!\varepsilon),\,\varepsilon\!>\!0$. Thus, $\mU\mU^T\!\!=\!\mIdent$.
\end{proof}
\label{sec:en_det}

\algblock{while}{endwhile}
%\algloopdefx[Loop]{Loop}[1]{\textbf{Loop} #1}
\algblock[TryCatchFinally]{try}{endtry}
\algcblockdefx[TryCatchFinally]{TryCatchFinally}{catch}{endtry}
	[1]{\textbf{except}#1}{}
\algcblockdefx[TryCatchFinally]{TryCatchFinally}{elsee}{endtry}
	[1]{\textbf{else}#1}{}
\algtext*{endwhile}
\algtext*{endtry}

\algblockdefx{ifff}{endifff}
	[1]{\textbf{if}#1}{}
\algtext*{endifff}

\algblockdefx{elseee}{endelseee}
	[1]{\textbf{else}#1}{}
\algtext*{endelseee}

%Based on these observations, we can readily extend MaxExp(F) to general high-order tensors with TSO as follows:
%\begin{equation}
%\mygthreeehat{TSO}{\tM^{(r)};\eta\,}\!=\!\tI_r\!-\left(\tI_r-\tM^{(r)}\right)^{\eta},
%\label{eq:tso1}
%\end{equation}
%where the identity tensor of order $r$ is defined as $\tI_r$ such that all elements $\tInb_{1,\cdots,1}\!=\!\tInb_{2,\cdots,2}\cdots\!=\tInb_{d,\cdots,d}\!=\!1$ and $\tInb_{i_1,\cdots,i_r}\!=\!0$ if $i_j\!\neq\!i_k$ and  $j\!\neq\!k,\,j,k\!\in\!\idx{r}$.

\begin{algorithm}[t]
\caption{Tensor Shrinkage Operator with Exponentiation by Squaring, left part for even orders and right part for odd orders $r$.}
\label{code:exp_sqr}
%\begin{subfigure}[t]{0.45\linewidth}{\fontsize{8}{2}{
\begin{minipage}{0.45\linewidth}{\fontsize{8}{2}{
{\bf Input:} $\tM$, $\eta\!\geq\!1$, $r\!=\!2,4,\cdots$
\begin{algorithmic}[1]
\State{$\tM^*_1\!\!=\!\tI_r\!-\!\tM,\; n\!=\!\text{int}(\eta),\;t\!=\!1,\;q\!=\!1$}
\while{ $n\!\neq\!0$:}
	\ifff{ $n\&1$:}
	\State \textbf{if }$t\!>\!1$: $\tG_{t+1}\!=\!\tG_{t}\!\times_{1,\cdots,r/2}\!\tM^*_{q},\!\!\!\!$\\$\qquad\quad$\textbf{else}: $\tG_{t+1}\!=\!\tM^*_{q}\!\!\!\!\!\!\!\!\!$
	\State $n\!\gets\!n\!-\!1,\;t\!\gets\!t\!+\!1$
	\endifff
	%\EndIf\State \textbf{end}
	\State{$n\!\gets\!\text{int}(n/2)$}
	\ifff{ $n\!>\!0$:}
	\State $\tM^*_{q+1}\!=\tM^*_{q}\times_{1,\cdots,r/2}\tM^*_{q}$
	\State $q\!\gets\!q\!+\!1$
	\endifff
\endwhile
\end{algorithmic}
{\bf Output:} $\mygthreeehat{TSO}{\,\tM}\!=\!\tI_r\!-\!\tG_t$}}
%\end{subfigure}
\end{minipage}
%\end{tcolorbox}
\hspace{.8cm}
%\begin{subfigure}[t]{0.45\linewidth}{\fontsize{8}{2}{
\begin{minipage}{0.45\linewidth}{\fontsize{8}{2}{
{\bf Input:} $\tM$, $\eta\!=3^0,3^1,3^2,\cdots$, $r\!=\!3,5,\cdots$
\begin{algorithmic}[1]
\State{$\tM^*_1\!\!=\!\tI_r\!-\!\tM,\; n\!=\!\text{int}(\eta),\;q\!=\!1$}
\while{ $n\!\neq\!0$:}
	\State{$n\!\gets\!\text{int}(n/3)$}
	\ifff{ $n\!>\!0$:}
	\State $\tM^*_{q+1}\!=\tM^*_{q}\times_{1,\cdots,\left\lfloor r/2 \right\rfloor}\tM^*_{q}$\\$\qquad\qquad\qquad\qquad\quad\times_{1,\cdots,\left\lceil r/2 \right\rceil}\tM^*_{q}$
	\State $q\!\gets\!q\!+\!1$
	\endifff
\endwhile
\end{algorithmic}
{\bf Output:} $\mygthreeehat{TSO}{\,\tM}\!=\!\tI_r\!-\!\tM^*_q$}}
%\end{subfigure}
\end{minipage}
\end{algorithm}

Note  $\tI_r\!-\!\left(\tI_r-\tM\right)^\eta\!\rightarrow\!\tI_r$ if $\eta\!\rightarrow\!\infty$. Thus, for sufficiently large  $1\!\ll\!\eta\!\ll\!\infty$, the diffused heat reverses towards the super-diagonal, where the majority of the signal should concentrate. Thus, we limit the number of coefficients of feature representations by extracting the super-diagonals from the TSO-processed $\tM^{(r)}$ as in Eq.~\ref{eq:tso_1}, where $r$ is the order of HoTD $\tM$. In our experiments, $r\!=\!2,3,4$.

As super-diagonals contain information obtained by multiplying $\eta$ times the complement $1\!-\!\lambda_i$ (spectral domain), %this can be simply realized by $\eta\!-\!1$ tensor products. Thus, 
TSO can be seen as $\eta\!-\!1$ aggregation steps along the tensor product mode(s). THus, we pass $\hat{\vpsi}_r$ via the element-wise SigmE from Eq.~\eqref{eq:sigme1} (as in Eq.~\eqref{eq:tso_2}) to detect the presence of at least one feature being detected in $\hat{\vpsi}_{r}$ after such an aggregation. %As patterns of the heat reversal are hard to trace, we cannot simply assume that $\eta'_r\!\approx\!N$, where $N$ is the number of feature vectors that were used to form $\tM^{(r)}$. Thus, we learn $\eta'_r$ via a simple attention modulator shown in Fig. \ref{fig:att1}.
For brevity, we drop subscript $r$ of $\vpsi_r$. % for brevity. In section~\ref{sec:ablations} we compare the performance for several values of $r$.

%\vspace{0.1cm}
%\noindent\textbf
\paragraph{Complexity.} For integers $\eta\!\geq\!2$ and even orders $r\!\geq\!2$, computing $\eta\!-\!1$ tensor-tensor multiplications $\big(\tI_r-\tM^{(r)}\big)^\eta$ has
the complexity $\bigoh\big(d^{\frac{3}{2}r}\eta\big)$. For odd orders $r\!\geq\!3$, due to alternations between multiplications in $\left\lfloor \frac{r}{2}\right\rfloor$ and $\left\lceil \frac{r}{2}\right\rceil$ modes,  the complexity is  $\bigoh\big(d^{\left\lfloor\frac{r}{2}\right\rfloor}d^{2\left\lceil\frac{r}{2}\right\rceil}\eta\big)\!\approx\!\bigoh\big(d^{\frac{3}{2}r}\eta\big)$. Thus, the complexity of Eq. \eqref{eq:tso1} \wrt integer $\eta\!\geq\!2$ scales linearly. However, for even orders $r$, one can readily replace $\big(\tI_r-\tM^{(r)}\big)^\eta$ with exponentiation by squaring \cite{expsquare}, whose cost is $\log({\eta})$. This readily yields the sublinear complexity $\bigoh\big(d^{\frac{3}{2}r}\log({\eta})\big)$ \wrt $\eta$. 

%\vspace{0.1cm}
%\noindent\textbf
\paragraph{Implementation of TSO.}  
Algorithm \ref{code:exp_sqr} %(See \S \ref{sec:algo} in \textbf{Suppl. Material}) 
shows  fast TSO for even/odd orders $r$. We restrict the odd variant to $r\!=\!3^0,3^1,3^2,\cdots$ for brevity. % but derivations of a more complete recurrent formula for $r\!=\!3,5,7,\cdots$ is simple. 
Finally, we note that matrix-matrix and tensor multiplications with cuBLAS are highly parallelizable so the $d^{\frac{3}{2}r}$ part of complexity can be reduced in theory even to $\log(d)$. 

%\vspace{0.1cm}\noindent\textbf
\paragraph{TENET-RPN.} To extract query RoIs, we firstly generate a set $\tZ$ of TSO representations $\vPsi\equiv\{\vpsi_z\}_{z\in\idx{Z}}$ from the support images, with $Z=|\tZ|$.
We then perform cross-attention between $\vPsi \in \mbr{d\times Z}$ and the $N$ feature vectors $\vPhi^*\! \in \mbr{d\times N}$ extracted from the query image. The attention input $\mQ$ is then generated from $\vPhi^*$, while $\mK$ and $\mV$ are both generated from $\vPsi$. The output of the transformer block in Eq.~\ref{eq:atten} is fed into an RPN layer to output a set $\tB$ of $B=|\tB|$ query RoIs.

Query RoIs are represented by spatially ordered features $\{\vPhi^*_b\}_{b\in\idx{B}}$ and TSO orderless features $\{\vpsi^*_b\}_{b\in\idx{B}}$,  passed to the   Transformer Relation Head. %, described below.

\subsection{Transformer Relation Head}
TRH, in Fig.~\ref{fig:embeddings},  models relations between support crops and query RoIs.
% 
%Spatially ordered representations $\{\vPhi_z\}_{z\in\idx{Z}}$ and $\{\vPhi^*_b\}_{b\in\idx{B}}$, and TSO orderless representations $\{\vpsi^*_b\}_{b\in\idx{B}}$ and $\{\vpsi_z\}_{z\in\idx{Z}}$ are the input to TRH. 
%
%
 %a good mix of spatially ordered and high-order spatially orderless descriptors
%TRH takes as input  the set of TSO representations $\{\vpsi^*_{b}\in \mbr{d}\}_{b\in \idx{B}}$ generated for each query RoI in $\tB$, and the set of TSO representations $\{\vpsi_{z}\in \mbr{d}\}_{z\in \idx{Z}}$ for support images in $\tZ$. 
%
TSO representations are derived from features of layer 4 of ResNet-50, leading to a channel dimension $d=1024$.
Spatially ordered representations of  support features $\{\vPhi_z\in\mbr{2d\times N}\}_{z\in\idx{Z}}$ and query RoIs $\{\vPhi^*_b\in\mbr{2d\times N}\}_{b\in\idx{B}}$ are both extracted from layer 5 of ResNet-50, leading to a channel dimension $2d=2048$. 

They are then fed into 2 different transformers: (i) a \textbf{$Z$-shot transformer head}, which performs cross-attention between globally-pooled representations of the query images and support images; and (ii) a \textbf{Spatial-HOP transformer head}, which performs self-attention between 
spatially ordered representations and spatially orderless high-order representations for a given image. %, be it a query or support image. We describe both transformer heads in more detail next.

%\vspace{0.1cm}\noindent\textbf
\paragraph{\textbf{$Z$-shot} transformer head}\label{z-shot-head}
 consists of a cross-attention layer formed with:
%\begin{align}
%& \mathbf{Q}=[\vq_1,\cdots,\vq_B] \quad\text{where}\quad \vq_b=\mW^{(q)}\big(\bar{\vphi}^*_b + \mW^{(p)}\vpsi^*_{b}\big),
%\label{eq:zsa_q}\\
%& \mathbf{K}=[\vk_1,\cdots,\vk_Z] \quad\text{where}\quad \vk_z=\mW^{(k)}\big(\bar{\vphi}_z + \mW^{(p)}\vpsi_{z}\big),
%\label{eq:zsa_k}\\
%& \mathbf{V}=[\vv_1,\cdots,\vv_Z] \quad\text{where}\quad \vv_z=\mW^{(v)}\big(\bar{\vphi}_z +\mW^{(p)}\vpsi_{z}\big),
%\label{eq:zsa_v}
%\end{align}
%
%
\begin{align}
&\left(\begin{array}{l}
\mathbf{Q}\\
\mathbf{K}\\
\mathbf{V}
\end{array}\right)\!=\!\left(
\begin{array}{l}
[\vq_1,\cdots,\vq_B]\\
\text{$[\vk_1,\cdots,\vk_Z]$}\\
\text{$[\vv_1,\cdots,\vv_Z]$}
\end{array}\right)
\quad\text{where}\quad
\left(
\begin{array}{l}
\vq_b\\
\vk_z\\
\vv_z
\end{array}\right)\!=\!
\left(
\begin{array}{l}
\mW^{(q)}\big(\bar{\vphi}^*_b + \mW^{(p)}\vpsi^*_{b}\big)\\
\mW^{(k)}\big(\bar{\vphi}_z + \mW^{(p)}\vpsi_{z}\big)\\
\mW^{(v)}\big(\bar{\vphi}_z +\mW^{(p)}\vpsi_{z}\big)
\end{array}
\right).
%& \mathbf{Q}=[\vq_1,\cdots,\vq_B] \quad\text{where}\quad \vq_b=\mW^{(q)}\big(\bar{\vphi}^*_b + \mW^{(p)}\vpsi^*_{b}\big),
%\label{eq:zsa_q}\\
%& \mathbf{K}=[\vk_1,\cdots,\vk_Z] \quad\text{where}\quad \vk_z=\mW^{(k)}\big(\bar{\vphi}_z + \mW^{(p)}\vpsi_{z}\big),
%\label{eq:zsa_k}\\
%& \mathbf{V}=[\vv_1,\cdots,\vv_Z] \quad\text{where}\quad \vv_z=\mW^{(v)}\big(\bar{\vphi}_z +\mW^{(p)}\vpsi_{z}\big),
%\label{eq:zsa_v}
\label{eq:zsa_q}
\end{align}
%
%\begin{cases}
%
%$\oplus$ denotes the element-wise addition, and
Moreover, $\bar{\vphi}^*_b$ and $\bar{\vphi}_z$ are average-pooled features $\frac{1}{N}\vPhi^*_b\mathbf{1}$ and $\frac{1}{N}\vPhi_z\mathbf{1}$, respectively. 
The matrices $\mW^{(q)}\in\mbr{2d\times 2d}$, $\mW^{(k)}\in\mbr{2d\times 2d}$, $\mW^{(v)}\in\mbr{2d\times 2d}$, and $\mW^{(p)}\in\mbr{2d\times d}$ is a linear projection mixing spatially orderless TSO representations  with spatially ordered representations. 
%The $\mW_p$ weights are shared by query RoI and support TSO representations and project such representations into a $2d$ space.
Thus, each attention query vector $\vq_b$ combines the extracted spatially orderless TSO representations  with spatially ordered representations for a given query RoI. Similarly, each key vector $\vk_z$ and value vector $\vv_z$ combine such two types of representations for a support crop. %The layer performs cross-attention, enhancing representations of the RoIs and support images with similar information.

%\vspace{0.1cm}\noindent\textbf
\paragraph{Spatial-HOP transformer head}\label{spatial-hop head}
 consists of a layer that performs self-attention on  spatially orderless TSO representations  and spatially ordered representations, extracted either from $Z$ support crops, or $B$ query RoIs. Below we take one support crop as an example. 
For the set $\tZ$, we compute ``Spatial'', a spatially ordered $Z$-averaged representation   $\vPhi^{\dagger}=\frac{1}{Z}\sum_{z\in\idx{Z}}\vPhi_z\in\mbr{2d\times N}$. We also compute HO, a spatially orderless High-Order $Z$-pooled representation  $\vpsi^{\dagger}=\frac{1}{Z}\sum_{z\in\idx{z}}\vpsi_z\in\mbr{d}$.
We split $\vPhi^{\dagger}$ along the channel mode of dimension $2d$ to create two new matrices $\vPhi^{\dagger u}\in\mbr{d\times N}$ and $\vPhi^{\dagger l}\in\mbr{d\times N}$. 
We set $\vPhi^{\dagger l}=[\vphi^{\dagger l}_{1},\cdots,\vphi^{\dagger l}_{N}]\in\mbr{d}$. Also, we form FO (from $\vPhi^{\dagger u}$), a spatially orderless First-Order $Z$-pooled representation.  Self-attention is then performed over the matrix of token vectors:  %$\tT=[\vphi^{\dagger l}_{1},\cdots,\vphi^{\dagger l}_{N}, \bar{\vphi}^{\dagger u}, \mW^{(g)}\vpsi^{\dagger}]$ 
\begin{equation}
\tT=[\vphi^{\dagger l}_{1},\cdots,\vphi^{\dagger l}_{N}, \bar{\vphi}^{\dagger u}, \mW^{(g)}\vpsi^{\dagger}],
\label{eq:shop_q}
\end{equation}
where $\mW^{(g)}\in\mbr{d\times d}$ a linear projection for HO. $\tT$ is projected onto the query, key and value linear projections. Spatial-HOP attention captures relations among spatially-aware first-order and spatially orderless high-order representations. % and local representations that have been pooled across all $Z$ shots.
 % except that  query RoIs are not $Z$-averaged. 
The outputs of  $Z$-shot and Spatial-HOP transformer heads are  combined and fed into the classifier/bounding-box regressor. %We compute tokens for each query RoI by analogy. 
See \S \ref{sec:TRH} of \textbf{Suppl. Material} for details.

%\clearpage\mbox{}Page \thepage\ of the manuscript.
%\clearpage\mbox{}Page \thepage\ of the manuscript.

%This is the last page of the manuscript.
%\par\vfill\par

\subsection{Pipeline Details (Figure \ref{fig:embeddings} ({\em bottom}))} \label{pipe-det}
%Figure \ref{fig:embeddings} ({\em bottom}) contains details of our pipeline.
%\begin{enumerate}
%\item 

%\vspace{0.1cm}
%\noindent\textbf
\paragraph{HOP unit}  uses  $\circledcirc$ to split the channel mode into groups (\eg, 2:1:1 split means two parts of the channel dimension are used to form $\tM^{(2)}$, one part to form $\tM^{(3)}$, and one part to form $\tM^{(4)}$). TSOs with parameters $\eta_2\!=\!\cdots\!=\!\eta_r\!=\!\eta$ are applied for orders $r\!=\!2,\cdots,r$, and diagonal entries are extracted from each tensor and concatenated by $\odot$. Element-wise SigmE with  $\eta'_2\!=\!\cdots\!=\!\eta'_r\!=\!\eta'$ is applied.

%\vspace{0.1cm}
%\noindent\textbf
\paragraph{``Orderless FO, Spatial, Orderless HO'' block} combines the First-Order (FO), spatial and High-Order (HO) representations. Operator ``Z-avg'' performs average pooling along $Z$-way mode, operator ``Sp-avg'' performs average pooling along the spatial modes of feature maps,  operator $\circledcirc$ simply splits the channel mode into two equally sized groups (each is half of the channel dimension), and operator $\circledast$ performs concatenation of FO, spatial and HO representations along the spatial mode of feature maps, \eg, we obtain $N\!+\!2$ fibers times 1024 channels.

%\vspace{0.1cm}
%\noindent\textbf
\paragraph{Z-shot T-RH} is a transformer which performs attention on individual Z-shots. The spatial representation $\mPhi$ is average pooled along spatial dimensions by ``Sp-avg'' and combined with high-order $\mPsi$ (passed by a FC layer) via addition $\oplus$. Another FC layer follows and subsequently the value, key and query matrices are computed,  an RBF attention formed. Operator $\bullet$ multiplies the value matrix with the  RBF matrix. Head is repeated $T$ times, outputs concatenated by $\odot$. % to form the output of this transformer.

%\vspace{0.1cm}\noindent\textbf
\paragraph{Spatial-HOP T-RH} takes  inputs from the ``Orderless FO, Spatial, Orderless HO'' block, and computes the value,  key and  query matrices. % per support and query RoIs. 
The attention matrix (RBF kernel) has $(N\!+\!2)\!\times\!(N\!+\!2)$ size, being composed of spatial, FO-spatial and HO-spatial attention. After multiplying the attention matrix with the value matrix, we extract the spatial, FO and HO representations. Support and query first-order representations (FO) (and high-order representations (HO)) are element-wisely multiplied by $\bullet$ (multiplicative relationship). %Support and query high-order representations (HO) also use the multiplicative relationship. 
Finally,  support and query spatial representations use the subtraction operator $\ominus$. After the concatenation of FO and HO relational representations by $\odot$, passing via an MLP (FC+ReLU+ FC), and concatenation with the spatial relational representations, we get one attention block, repeated $T$ times.

%\end{enumerate}

\section{Experiments}
\label{sec:exp}
%Below, we evaluate  KFSOD   on  popular benchmarks. %  to compare it with the state of the art. 

%\vspace{-0.05cm}
%\noindent\textbf
\paragraph{Datasets and Settings}. 
For  PASCAL VOC 2007/12 \cite{voc}, we adopt the 15/5 base/novel category split setting and use training/validation sets from PASCAL VOC 2007 and 2012 for training, and the testing set from PASCAL VOC 2007 for testing, following \cite{f10ObjectDetection,fsod,accv,persampel}. For MS COCO \cite{mscoco}, we follow \cite{f12}, and adopt the 20 categories that overlap with PASCAL VOC as the novel categories for testing, whereas the remaining 60 categories  are used for training. For the FSOD dataset \cite{fsod}, we split its 1000 categories into 800/200 for training/testing.

%\vspace{0.1cm}
%\noindent\textbf
\paragraph{Implementation Details.} TENET uses ResNet-50 pre-trained on ImageNet \cite{imagenet} and MS COCO \cite{mscoco}. 
We fine-tune the network with a learning rate of 0.002 for the first 56000 iterations and 0.0002 for another 4000 iterations. 
Images are resized to 600 pixels (shorter edge) and the longer edge is capped at 1000 pixels.  Each support image is cropped based on ground-truth boxes, bilinearly interpolated and padded to $320\!\times\!320$ pixels. 
%implemented in PyTorch.
%All tansformer in TENET are set to have 4-heads, 2-blocks with layer normalization, and ReLU is used as an activation function. The dropout rate of Transformer is 0.1.
We set, via cross-validation) SigmE parameter $\eta'\!\!=\!\!200$ and TSO parameters $\eta_2\!=\!\eta_3\!=\!\eta_4\!=\!7$. 
We report standard metrics for FSOD, namely $mAP$, $AP$, $AP_{50}$ and $AP_{75}$.

\definecolor{LightCyan}{rgb}{0.88,1,0.88}

\setlength{\tabcolsep}{0.8pt}
\begin{table}[t]
%\centering
\fontsize{7}{7}\selectfont  
\caption{Evaluations (mAP \%) %of novel classes using
on three splits of the VOC 2007 testing set.}
\label{VOCsplit}
\hspace{-0.5cm}
\begin{tabular}{ll|cccc|cccc|cccc|cccc}
\toprule
\multicolumn{2}{c}{\multirow{2}{*}{Method/Shot}}&  
\multicolumn{4}{c}{Split 1}&  
\multicolumn{4}{c}{Split 2}&  
\multicolumn{4}{c}{Split 3}&
\multicolumn{4}{c}{Mean}\\
%\cline{3-18}
\cmidrule{3-18}
&\multicolumn{1}{c}{}&  1 &3 &5&\multicolumn{1}{c}{10}& 1 &3 &5&\multicolumn{1}{c}{10}& 1 &3 &5&\multicolumn{1}{c}{10}& 1 &3 &5&\multicolumn{1}{c}{10}\\
%\hline
\midrule
FRCN &ICCV12  & 11.9&29.0&36.9&36.9&5.9&23.4&29.1&28.8&5.0&18.1&30.8&43.4&7.6&23.5&32.3&36.4\\ %& \cite{re32}
FR &ICCV19 &14.8 &26.7 &33.9& 47.2& 15.7 &22.7 &30.1 &39.2 &19.2 &25.7& 40.6& 41.3&16.6&25.0&34.9&42.6\\%&\cite{f10ObjectDetection}
Meta& ICCV19 &19.9 &35.0& 45.7 &51.5& 10.4 &29.6 &34.8 &45.4 &14.3& 27.5& 41.2& 48.1&14.9&30.7&40.6&48.3\\%&\cite{f12}
FSOD& CVPR20&29.8 &36.3&48.4 &53.6& 22.2 & 25.2&31.2 &39.7&24.3& 34.4&47.1& 50.4&25.4&32.0&42.2&47.9\\% &\cite{fsod}
NP-RepMet& NeurIPS20&37.8 &41.7 &47.3 &49.4 &{\bf 41.6} &43.4 &47.4 &49.1& 33.3& 39.8 & 41.5 &44.8&37.6&41.6&45.4&47.8\\% & \cite{neg} 

PNSD &ACCV20& 32.4&39.6&50.2 &55.1&30.2 &30.3 &36.4 &42.3 &30.8 &38.6 &46.9 &52.4 &31.3&36.2&44.5&49.9\\
MPSR &ECCV20 &41.7 &51.4 &55.2 &61.8& 24.4 &39.2& 39.9 &47.8& {\bf 35.6} & 42.3 &48.0 &49.7&33.9&44.3&47.7&53.1\\%&\cite{multi}
TFA &ICML20& 39.8&44.7 &55.7 &56.0  &23.5  &34.1 &35.1 &39.1 &30.8 &42.8 &49.5&49.8&31.4&40.5&46.8&48.3\\
FSCE &CVPR21& 44.2&51.4 & 61.9 &63.4  &27.3  &43.5 &44.2 &50.2 &22.6 &39.5 &47.3&54.0&31.4&44.8&51.1&55.9\\
CGDP+FRCN &CVPR21& 40.7&46.5 &57.4 &62.4  &27.3  &40.8 &42.7 &46.3 &31.2 &43.7 &50.1&55.6&33.1&43.7&50.0&54.8\\
TIP &CVPR21& 27.7&43.3 &50.2 &56.6  &22.7  &33.8 &40.9 &46.9 &21.7 &38.1 &44.5&50.9&24.0&38.4&45.2&52.5\\
FSOD\textsuperscript{up}&ICCV21& 43.8&50.3 & 55.4 &61.7  &31.2  &41.2 &44.2 &48.3 &35.5 &43.9 &50.6&53.5&36.8&45.1&50.1&54.5\\
QSAM&WACV22& 31.1&39.2 & 50.7 &59.4  &22.9  &32.1 &35.4 &42.7 &24.3 &35.0 &50.0&53.6&26.1&35.4&45.4&51.9\\
\midrule
\rowcolor{LightCyan}TENET&(Ours)&{\bf 46.7} &{\bf 55.4} &\bf62.3 &{\bf 66.9} &40.3 &\bf44.7 &{\bf 49.3} &  \bf{52.1} &35.5 &{\bf 46.0} &{\bf 54.4} & \bf54.6 &{\bf 40.8}&{\bf 48.7}&{\bf55.3}&\bf{57.9}\\
\bottomrule
\end{tabular}
%\vspace{-0.5cm}
\end{table}

%\vspace{-0.1cm}
\subsection{Comparisons with the State of the Art}
%\vspace{-0.1cm}
%\noindent\textbf
\paragraph{PASCAL VOC 2007/12.} We compare our method to QSAM\cite{persampel}, FSOD\textsuperscript{up} \cite{fsodup}, CGDP+FRCN \cite{cgdp}, TIP \cite{tip}, FSCE \cite{FSCE}, TFA \cite{TFA}, Feature Reweighting (FR) \cite{f10ObjectDetection}, LSTD \cite{f9LSTD}, FRCN \cite{re32}, NP-RepMet \cite{neg}, MPSR\cite{multi}, PSND \cite{accv} and FSOD \cite{fsod}. %\footnote{FSOD \cite{fsod} did not use PASCAL VOC 2007  (we evaluate it ourselves).}. 
Table \ref{VOCsplit} shows that our TENET  outperforms FSOD  by a 7.1--15.4\% margin. For the 1- and 10-shot  regime, we outperform QSAM \cite{persampel}  by $\sim$14.7\%. %Even on the 10-shot split, we outperform FSOD\textsuperscript{up}  by 2.2\%. 
%
%Table \ref{VOC} (\S \ref{sec:voc_classwise} of \textbf{Suppl. Material}) shows class-wise  results (5-shot protocol): TENET gains 11.2\% and 6.5\% mAP (novel and base classes) over PNSD. %Despite good pe FRCN \cite{re32} under the few-shot setting: without adequate training images, it  detects poorly  objects from novel classes. In contrast, 
%Our KFSOD provides best results on both base and  novel-classes. 

%\item[ii)]
%\vspace{0.1cm}
%\noindent\textbf
\paragraph{MS COCO.} Table \ref{coco} compares TENET with QSAM\cite{persampel}, FSOD\textsuperscript{up}  \cite{fsodup},  CGDP+ FRCN \cite{cgdp}, TIP \cite{tip}, FSCE \cite{FSCE}, TFA \cite{TFA}, FR \cite{f10ObjectDetection}, Meta R-CNN \cite{f12}, FSOD \cite{fsod} and PNSD\cite{accv} on  the MS COCO minival set (20 novel categories, 10-shot protocol). Although MS COCO is more challenging in terms of complexity and the dataset size, TENET  boosts results to 19.1\%, 27.4\% and 19.6\%, surpassing the  SOTA method QSAM by 6.1\%, 2.7\% and 7.5\% on $AP$, $AP_{50}$ and $AP_{75}$. 

%\vspace{0.1cm}
%\noindent\textbf
\paragraph{FSOD.} Table \ref{fsod} compares TENET (5-shot) with PNSD \cite{accv}, FSOD \cite{fsod}, LSTD \cite{f9LSTD} and LSTD (FRN \cite{re32}). We re-implement BD\&TK, modules of LSTD, based on Faster-RCNN for fairness. TENET yields SOTA 35.4\% $AP_{50}$ and 31.6\% $AP_{75}$. %Note that LSTD has to transfer knowledge from the source  to target domain by retraining on novel categories while our approach, \cite{accv} and \cite{fsod} are directly applied to detect novel categories.
%\end{enumerate}
\begin{table}[t]
%\vspace{-0.8cm}
\makeatletter\def\@captype{table}\makeatother\caption{Evaluations on the MS COCO minival set (\ref{coco}) and FSOD testset (\ref{fsod}).}
%\vspace{-0.25cm}
\centering
%\hspace{-0.23cm}
\begin{subfigure}[t]{0.52\linewidth}{
\centering
\fontsize{7}{6.0}\selectfont  
\centering
\setlength{\tabcolsep}{1pt}{
\begin{tabular}{cccccc}
%\hline
\toprule
Shot & \multicolumn{2}{c}{ Method} &$AP$&$AP_{50}$&$AP_{75}$\\
\midrule
\multirow{9}{*}{10}& LSTD&AAAI18 & 3.2	& 8.1&2.1 \\

& FR & ICCV12& 5.6	& 12.3&4.6 \\

&Meta&ICCV19 &8.7	&19.18	&6.6 \\
&MPSR&ECCV20 &9.8&17.9&9.7\\
&FSOD &CVPR20&  11.1 & 20.4 & 10.6\\
&PNSD&ACCV20 &  12.3 & 21.7 & 11.7\\
&TFA&ICML20 &  9.6 & 10.0 & 9.3\\
&FSCE&CVPR21 &  10.7& 11.9 & 10.5\\
&CGDP+FRCN&CVPR21 &  11.3& 20.3 & 11.5 \\
&FSOD\textsuperscript{up}&ICCV21 &  11.6& 23.9 & 9.8\\
&QSAM&WACV22 &  13.0&24.7 & 12.1\\
\midrule
\rowcolor{LightCyan}&TENET&(Ours)&\textbf{19.1} & \textbf{27.4} & \textbf{19.6}\\
\bottomrule
\end{tabular}}}
%\vspace{-0.3cm}
\caption{\label{coco}}
\end{subfigure}
%\vspace{0.55cm}
%\hspace{1.cm}
%
\begin{subfigure}[t]{0.41\linewidth}{%\centering
%\fontsize{6}{14.3}
\fontsize{7}{12.4}\selectfont 
\centering
\setlength{\tabcolsep}{2.8pt}{
\begin{tabular}{ccccc}
%\hline
\toprule
Shot &  \multicolumn{2}{c}{ Method}  &$AP_{50}$&$AP_{75}$\\
\midrule
\multirow{4}{*}{5}& $\substack{\text{LSTD}\\\text{(FRN)}}$ &AAAI18& 23.0	& 12.9 \\
&LSTD&AAAI18 &24.2	&13.5 \\
&FSOD&CVPR20&  27.5 & 19.4 \\
&PNSD&ACCV20 &  29.8 & 22.6\\
&QSAM&WACV22 &  30.7&25.9 \\
\midrule
\rowcolor{LightCyan}&TENET&(Ours)&\textbf{35.4} & \textbf{31.6} \\
\bottomrule
\end{tabular}}}
%\vspace{-0.2cm}
\caption{\label{fsod}}
\end{subfigure}
%\vspace{-1.5cm}
\end{table}

%\vspace{-0.58cm}
\subsection{Hyper-parameter and ablation analysis}
\label{sec:ablations}

\noindent{\textbf{TENET.}} Table \ref{Multi-order} shows that among orders $r\!=\!2$, $r\!=\!3$ and $r\!=\!4$,  variant $r\!=\!2$ is the best. %, unsurprisingly, the most informative one. 
We next consider pairs of orders, and the triplet $r\!=\!2,3,4$. As the number of tensor coefficients  grows quickly \wrt $r$, we split the 1024 channels into groups, \eg, $r\!\!=2,3$. A 3:1 split means that second- and third-order tensors are built from $768$ and $256$ channels ($768\!+\!256\!=\!1024$). We report only the best splits. For pairs of orders, variant $r\!=\!2,3$ was the best.  Triplet $r\!=\!2,3,4$, the best performer, outperforms $r\!=\!2$ by 5.8\% and 2.7\% in novel classes (5- and 10-shot), and 2.1\% and 2.4\% in base classes. As all representations are 1024-dimensional, we conclude that multi-order variants are the most informative.
%Table \ref{Multi-order} also shows speeds (millisecond) per image on inference. 

\begin{table}[t]
%\vspace{-0.8cm}
\makeatletter\def\@captype{table}\makeatother\caption{Results on VOC2007 testset for applying TENET in RPN or TRH (\ref{Multi-order}, top panel of \ref{STF}). TRH ablation shown in bottom panel of \ref{STF}. }
%\vspace{-0.25cm}
\centering
%\begin{subfigure}[b]{0.1\linewidth}
%\centering
%\end{subfigure}
\begin{subfigure}[b]{0.48\linewidth}{
\centering
\fontsize{6}{9}\selectfont  
\setlength{\tabcolsep}{1pt}
%\resizebox{\textwidth}{!}
{\begin{tabular}{ccc|c|cc cc c}
\toprule
\multicolumn{3}{c|}{$r$} &\multirow{2}{*}{$\substack{\text{dim.}\\\text{\fontsize{4}{1}\selectfont split}}$}& \multicolumn{2}{c}{Shot(Novel)} & \multicolumn{2}{c}{Shot(Base)} & Speed \\
\cmidrule{1-3}
\cmidrule{5-9}
2&3&4&~&5&10&5&10&(img/ms)\\
\midrule
\checkmark&~&~&~&56.5&64.2&71.7&75.5&32\\
~&\checkmark&~&~&55.7&63.2&67.0&72.1&69\\
~&~&\checkmark&~&51.4&58.9&68.7&74.8&78\\
\rowcolor{LightCyan}\checkmark&\checkmark&~&{\bf 3:1}&58.3&63.2&69.3&75.1&42\\
\checkmark&~&\checkmark&3:1&56.1&62.4&70.8&75.4&68\\
&\checkmark&\checkmark&2:2&51.8&61.7&68.1&73.6&71\\
\midrule
\multirow{6}{*}{\checkmark}&\multirow{6}{*}{\checkmark}&\multirow{6}{*}{\checkmark}&6:1:1&53.6&62.7&69.4&72.8&\\
\rowcolor{LightCyan}&&&{\bf 5:2:1}&\bf 62.3&\bf 66.9&\bf 73.8&\bf 77.9&59\\
&&&5:1:2&53.9&63.1&69.7&73.3&\\
&&&4:2:2&61.4&65.0&70.4&74.9&\\
&&&4:3:1&59.1&63.6&71.8&75.2&\\
&&&4:1:3&61.0&64.1&68.9&72.5&\\
\bottomrule
\end{tabular}}}
\caption{\label{Multi-order}}
\end{subfigure}
%\hspace{0.50cm}
\begin{subfigure}[b]{0.43\linewidth}
\centering
%\vspace{-2.70cm}
{\fontsize{6}{10}\selectfont  
\setlength{\tabcolsep}{3.2pt}{
\begin{tabular}{c|c|c|cc cc}
\toprule
\multirow{2}{*}{}&RPN &TRH& \multicolumn{2}{|c}{Shot(Novel)} & \multicolumn{2}{c}{Shot(Base)} \\
\cmidrule{2-7}
&\multicolumn{2}{c|}{$r$}&5&10&5&10\\
\midrule
a&1& 1&53.4&61.8&64.9&72.1\\
b&2,3,4&1&57.2&63.7&68.8&76.6\\
c&2,3,4&2,3,4&61.0&65.4&71.3&77.3\\
\midrule
\rowcolor{LightCyan}d&2,3,4&1,2,3,4&\bf 62.3&\bf 66.9&\bf 73.8&\bf 78.2\\
\bottomrule
\end{tabular}}
\fontsize{6}{10}\selectfont  
\setlength{\tabcolsep}{1.4pt}{
\begin{tabular}{c|c|cc cc}
\toprule
\multicolumn{2}{c}{TRH}& \multicolumn{2}{|c}{Shot(Novel)} & \multicolumn{2}{c}{Shot(Base)} \\
\midrule
$Z$-shot &Spatial-HOP &5&10&5&10\\
\midrule
\checkmark&&58.5&63.2&69.3&75.1\\
&\checkmark&61.0&65.8&71.7&76.5\\
\rowcolor{LightCyan}\checkmark&\checkmark&\bf 62.3&\bf 66.9&\bf 73.8&\bf 78.2\\
\bottomrule
\end{tabular}}}
\caption{\label{STF}}
\end{subfigure}
%\vspace{-0.8cm}
\vspace{-0.6cm}
\end{table}

%\vspace{0.1cm}
%\noindent\textbf
\paragraph{TSO.}
Based on the best channel-wise splits in Table \ref{Multi-order}, we study the impact of  $\eta_r$ (shrinkage/decorrelation) of TSO to verify its effectiveness. Figure \ref{sp} shows mAP \wrt the individual $\eta_2,\eta_3$ and $\eta_4$ for $r\!=\!2$, $r\!=\!3$ and $r\!=\!4$. %We note that the best $\eta_2,\eta_3$ and $\eta_4$ match those selected by crossvalidation on the validation split. 
We then investigate the impact of $\eta_r$ on pairwise representations, where we set the same $\eta_r$ for pairwise variants, \eg,  $\eta_2=\eta_3$. Again, the same value of $\eta_r$ is used for triplet $r\!=\!2,3,4$. 
Note that for $\eta_r\!=\!1$, TSO is switched off and all representations reduce to the polynomial feature maps in So-HoT \cite{me_domain}. As shown in Figure \ref{sp}, TSO is very beneficial ($\sim$ 5\% gain for triplet $r\!=2,3,4$ over not using TSO). 

%\vspace{0.1cm}
%\noindent\textbf
\paragraph{TRH.} Below we investigate the impact of $Z$-shot and Spatial-HOP T-RHs on results. Table \ref{STF} (bottom) shows that both heads
are highly complementary. % each other to produce higher performance.

%\vspace{0.1cm}
%\noindent\textbf
\paragraph{Other hyperparameters.}
%We first  examine the influence of RBF kernel parameter $\sigma$. 
We start by varying $\sigma$ of the RBF kernel from 0.3 to 3. Fig \ref{sigma} shows that $\sigma\!=\!0.5$ gives the best  result.
We now fix $\sigma$ and investigate the impact of varying the number of heads used in T-Heads Attention ($TA$). Table \ref{head} shows best performance with $TA\!=\!4$.
%together with $\sigma\!=\!0.5$ for which we show mAP \wrt $TB$ (Table \ref{block}). 
%It is obvious that increasing $TA$ within a certain range,  TENET achieves better result ($\sim$ 2\%/4\% improvement for four groups of MHA over a single group on 1/5 shot setting, novel classes).  
Lastly, we vary the number of TENET blocks ($TB$). Table \ref{block} shows that results are stable especially if $TB\geq 2$.
Unless otherwise noted, $TA=2$ and $TB=4$, respectively, on VOC dataset.  See \S \ref{sup_more} of \textbf{Suppl. Material} for more results on FSOD and MS COCO.

% \noindent{\textbf{Capacity.}} 
% In our TENET,  the following components can affect its  capacity:  i) the parameter $\sigma$  of  Gaussian kernel (self-attention function) %can control the model complexity, %\eg, a small (resp.  large)  radius  captures  a  complex  (resp.  simple)  decision  boundary;  
% ii) more/few splitted MHA groups ($TA$) and stacked TENET blocks ($TB$) %determine model' capacity.  
% In practice, we first  examine how the parameter $\sigma$ influences FSOD performance, where the $\sigma$  ranging from 0.3 to 3.  The $\sigma\!=\!0.5$ seems to give the best  performance,  as shown in Fig \ref{sigma}.  We  then fix $\sigma$ and proceed to investigate the impact of varying the groups of MHA (Table \ref{head}).  Again, best value ($TA\!=\!4$) together with $\sigma\!=\!0.5$ for which we show mAP \wrt $TB$ (Table \ref{block}). It is obvious that increasing $TA$ within a certain range,  TENET
% achieves better result ($\sim$ 2\%/4\% improvement for four groups of MHA over a single group on 1/5 shot setting, novel classes).  And our model is less sensitive to the block number when it is greater than 2. Unless otherwise noted,  the number of group and block are set to 4 and 2 on VOC dataset, respectively.  In our suppl. material, we present more results and discussions on FSOD and MS COCO dataset.

%\vspace{0.1cm}
%\noindent\textbf
\paragraph{Impact of TENET on RPN and TRH.} Table \ref{STF} shows ablations \wrt TENET variants in: 1) either RPN or TRH, or 2) both RPN and TRH. 
Comparing results for settings a,b, and c confirms that using second-, third- and fourth-orders simultaneously benefits both RPN and TRH, achieving 3.8\%/1.9\% as well as 3.8\%/1.8\% improvement on novel classes over the first-order-only variant. 
Results for settings c and d show that TRH  encodes better the information carried within regions if leveraging both first- and higher-order representations.
%To verify if the benefit comes from orders $r\!=\!2,3,4$ alone, or from the combination of orders $r\!=\!2,3,4$ with our TSO,  we investigate $r\!=\!\widetilde{2,3,4}$, a triplet variant without TSO. Table \ref{STF} shows that TSO brings up to 10\% mAP gain.
%
%
%
\begin{figure}[t]
%\hspace{0.8cm}
\centering
%trim={<left> <lower> <right> <upper>}
\begin{subfigure}[t]{0.49\linewidth}
\centering
\includegraphics[trim=24 23 25 25, clip=true, width=5.cm]{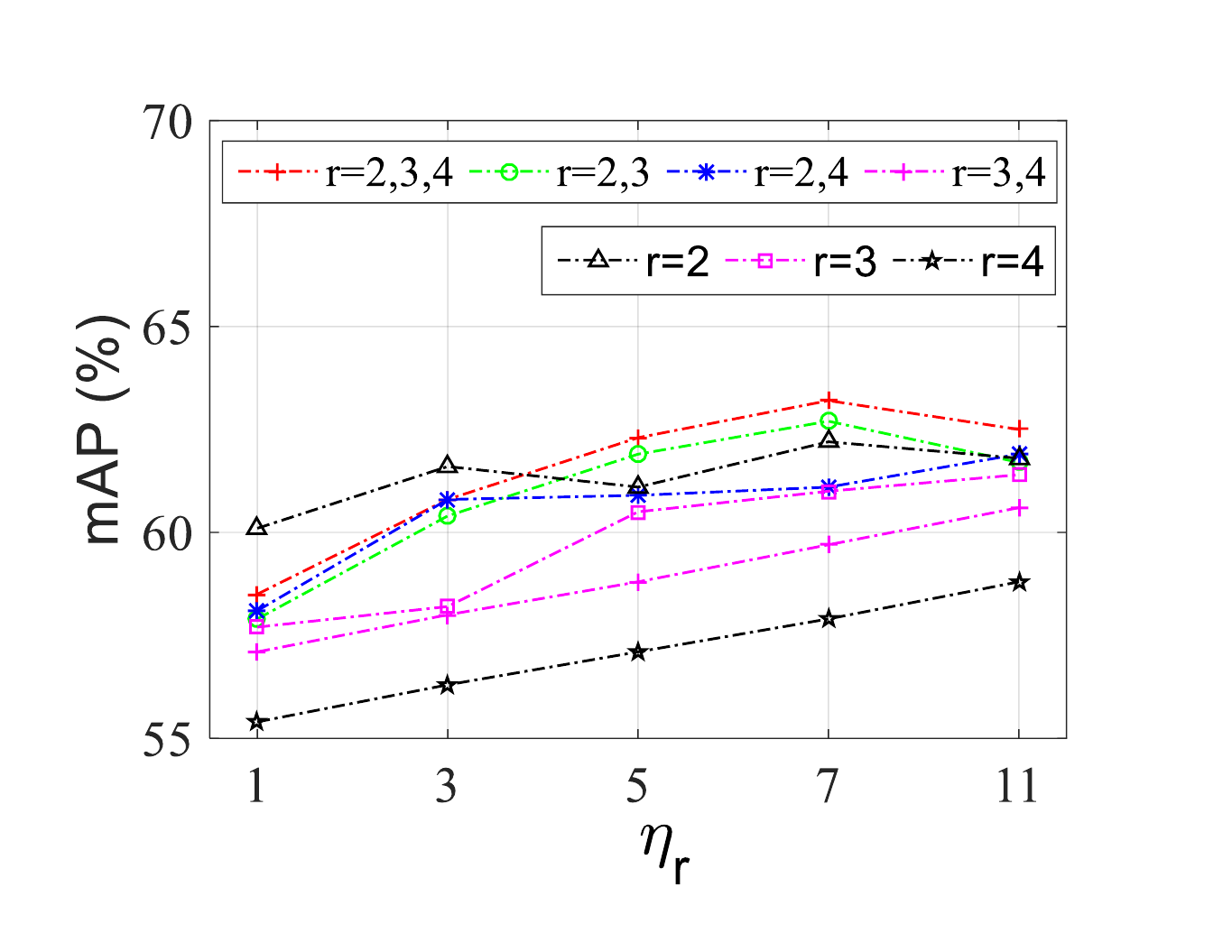}\caption{\label{sp}}
\end{subfigure}
\begin{subfigure}[t]{0.49\linewidth}
\centering
\includegraphics[trim=24 24 25 25, clip=true, width=5.cm]{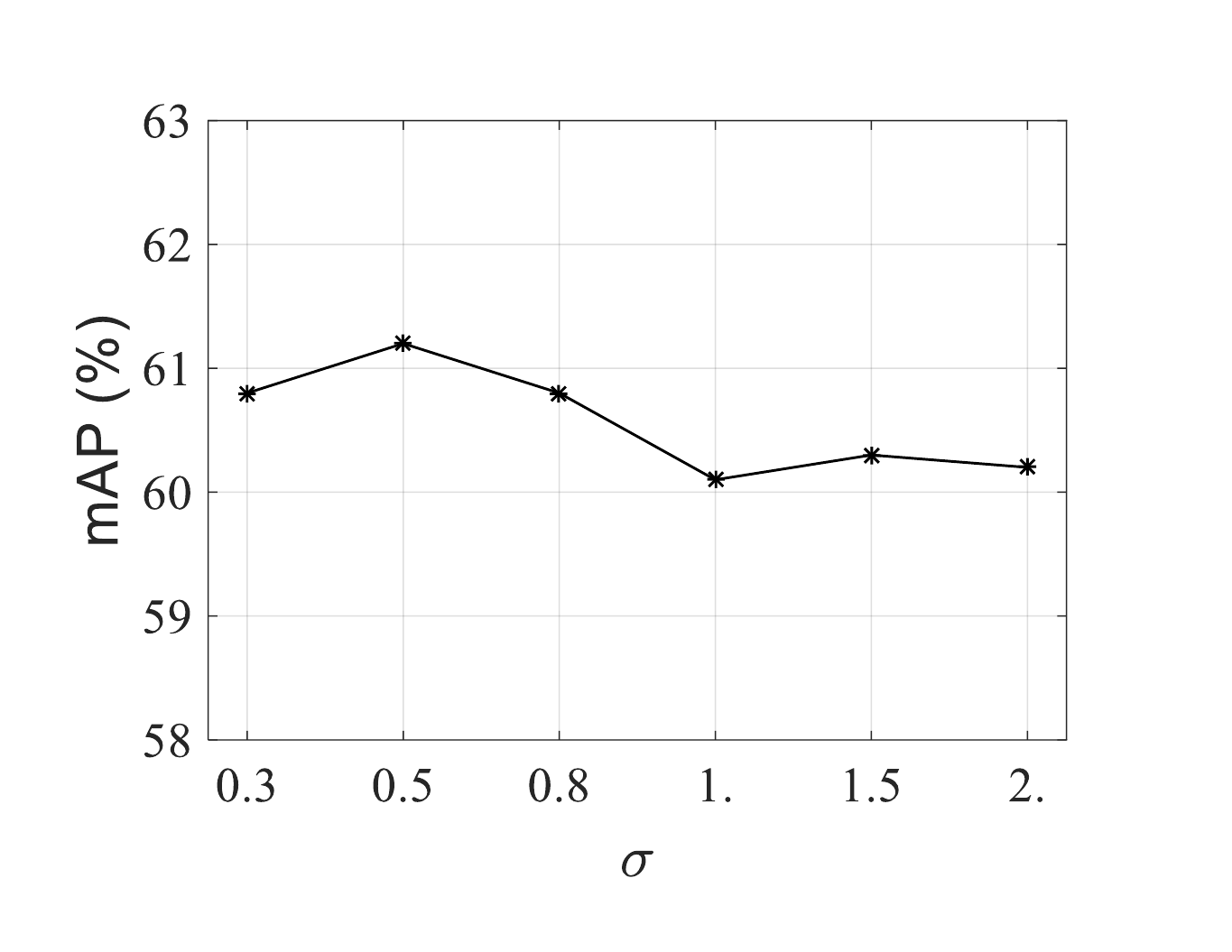}\caption{\label{sigma}}
\end{subfigure}
\caption{mAP (VOC2007 dataset, novel classes, 10-shot) \wrt varying $\eta_r$ in TSO (Fig. \ref{sp}) and the $\sigma$ of   RBF kernel in self-attention   (Fig. \ref{sigma}).}
\label{sp_pn}
%\vspace{-0.6cm}
\end{figure}
%
%
%%\cellcolor{LightCyan}\rowcolor{LightCyan}
\newcolumntype{a}{>{\columncolor{LightCyan}}c}
\begin{table}[t]
%\vspace{-0.1cm}
\makeatletter\def\@captype{table}\makeatother\caption{
Effect of varying (a) group within MHA in Tab. \ref{head} and (b) TENET block in Tab.  \ref{block}  on PASCAL VOC 2007 (5/10-shot, novel classes). When varying TENET block, group number is fixed to 4 (best value).}
%\vspace{-0.3cm}
\centering
\begin{subfigure}[t]{0.54\linewidth}
\centering
\fontsize{8}{10}\selectfont  
\setlength{\tabcolsep}{2pt}{
\begin{tabular}{c c|ccccccc}
\toprule
\multicolumn{2}{c|}{$TA$}&1&2& \cellcolor{LightCyan}4&8&16&32&64\\
%\midrule
%\addlinespace[2pt]
\arrayrulecolor{LightCyan}
\addlinespace[-3pt]
\cmidrule[4pt]{5-5}
\addlinespace[-3pt]
\arrayrulecolor{black}
\cline{1-9}
\arrayrulecolor{LightCyan}
\addlinespace[-1pt]
\cmidrule[2pt]{5-5}
\addlinespace[-3pt]
\arrayrulecolor{black}
%\arrayrulecolor{LightCyan} \specialrule{6pt}{0pt}{-6pt}\arrayrulecolor{black}
\multirow{2}{*}{$\substack{\text{Shot}\\\text{\fontsize{4}{1}\selectfont(Novel)}}$}&5&58.3&59.1& \cellcolor{LightCyan} \bf61.8&60.5&58.4&58.5&56.0\\
%\cmidrule{2-9}
%
\arrayrulecolor{LightCyan}
\addlinespace[-3pt]
\cmidrule[4pt]{5-5}
\addlinespace[-3pt]
\arrayrulecolor{black}
\cline{2-9}
\arrayrulecolor{LightCyan}
\addlinespace[-1pt]
\cmidrule[2pt]{5-5}
\addlinespace[-3pt]
\arrayrulecolor{black}
&10&61.2&62.8&\cellcolor{LightCyan}\bf65.8&64.2&61.2&61.7&60.4\\
\addlinespace[1pt]
\bottomrule
\end{tabular}
\caption{\label{head}}}
\end{subfigure}
\begin{subfigure}[t]{0.44\linewidth}
\centering
\fontsize{8}{10}\selectfont  
\setlength{\tabcolsep}{3pt}{
\begin{tabular}{c c|ccccc}
\toprule
\multicolumn{2}{c|}{$TB$}&1&\cellcolor{LightCyan}2&3&4&5\\
%\midrule
%
\arrayrulecolor{LightCyan}
\addlinespace[-3pt]
\cmidrule[4pt]{4-4}
\addlinespace[-3pt]
\arrayrulecolor{black}
\cline{1-7}
\arrayrulecolor{LightCyan}
\addlinespace[-1pt]
\cmidrule[2pt]{4-4}
\addlinespace[-3pt]
\arrayrulecolor{black}
\multirow{2}{*}{$\substack{\text{Shot}\\\text{\fontsize{4}{1}\selectfont(Novel)}}$}&5&61.8&\cellcolor{LightCyan}\bf 62.3&62.1&61.8&61.9\\
%\cmidrule{2-7}
%
%
\arrayrulecolor{LightCyan}
\addlinespace[-3pt]
\cmidrule[4pt]{4-4}
\addlinespace[-3pt]
\arrayrulecolor{black}
\cline{2-7}
\arrayrulecolor{LightCyan}
\addlinespace[-1pt]
\cmidrule[2pt]{4-4}
\addlinespace[-3pt]
\arrayrulecolor{black}
&10&65.8&\cellcolor{LightCyan}\bf  66.9&66.4&66.1&66.4\\
\addlinespace[1pt]
\bottomrule
\end{tabular}
\caption{\label{block}}}
\end{subfigure}
\vspace{-0.5cm}
\end{table}

\vspace{-0.3cm}
\section{Conclusions}
%\vspace{-0.1cm}
We have proposed TENET, which uses higher-order tensor descriptors, in combination with a novel Tensor Shrinkage Operator, to generate highly-discriminative representations with tractable dimensionality.
We use these representations in our proposed Transformer Relation Head to dynamically extract correlations between query image regions and  support crops.  TENET has heightened robustness to large intra-class variations, leading to SOTA performance on all benchmarks.

% in the FSOD setting in two novel ways. First, it uses higher-order tensor descriptors in combination with a novel tensor shrinkage operator (TSO). TSO, in analogy with heat diffusion reversal, concentrates information along the super-diagonal of tensor inputs. The extracted super-diagonal is then used as highly discriminative representations.
% Second, it uses transformers to dynamically extract correlations between query image regions and the entire support set for a class. The proposed Transformer Relation Head (TRH), sensitive to both  the variability between $Z$ support samples, and to positional variability among support and query objects, has a strong discriminative ability to distinguish and localize different classes, as shown empirically.
%In our \textbf{suppl. material}, we present additional results and discussions.

\vspace{0.1cm}
\noindent\textbf{Acknowledgements.} We thank Dr. Ke Sun for early discussions on TSO and its relation to the Kullback-Leibler divergence
and the Tsallis entropy. PK was supported by the CSIRO’s Machine Learning and Artificial Intelligence Future Science Platform (MLAI FSP).
\clearpage

\bibliographystyle{splncs04}
\bibliography{egbib}

%\processdelayedfloats

%\usepackage[nolists,nomarkers,figuresfirst]{endfloat}

%% LaTeX2e file `./5009-support.tex'
%% generated by the `filecontents' environment
%% from source `5009-arxiv' on 2022/10/31.
%%
\newpage
\appendix

\title{Time-rEversed diffusioN tEnsor Transformer:\\A new TENET of Few-Shot Object Detection (Supplementary Material)} % Replace with your title

\begin{comment}
\titlerunning{ECCV-22 submission ID \ECCVSubNumber}
\authorrunning{ECCV-22 submission ID \ECCVSubNumber}
\author{Anonymous ECCV submission}
\institute{Paper ID \ECCVSubNumber}
\end{comment}

\author{Shan Zhang$^{\star, \dagger}$\orcidlink{0000-0002-5531-3296} \and
Naila Murray$^{\clubsuit}$\orcidlink{0000-0001-7032-0403} \and
Lei Wang$^{\vardiamond}$\orcidlink{0000-0002-0961-0441} \and
Piotr Koniusz$^{\star,\S,\dagger}$\orcidlink{0000-0002-6340-5289}}
\authorrunning{Zhang \etal}
\titlerunning{Time-rEversed diffusioN tEnsor Transformer (TENET)}
\institute{$^{\dagger}$Australian National University \;
$^{\clubsuit}$Meta AI \\
   $^{\vardiamond}$University of Wollongong \;
   $^\S$Data61/CSIRO\\
   %\tt\small
   $^{\dagger}$firstname.lastname@anu.edu.au, $^{\vardiamond}$leiw@uow.edu.au,  $^{\clubsuit}$murrayn@fb.com
}

\maketitle
%We would like to stress that \textbf{\color{teal}we will release our code for the computer vision community}.
\setcounter{table}{4}
\setcounter{equation}{16}
\setcounter{figure}{2}
Below are additional derivations, evaluations and illustrations of our method.

\section{Ablation Study on Encoding Network}\label{backbone}

Below we perform  ablations of the backbone (Encoding Network, termed as EN in main paper). We use ConvNet (ResNet-50) and Transformer network \citelatex{swin_sup} (Swin-B$^7$/ Swin-B$^{12}$ pre-trained on ImageNet-22K \cite{imagenet} with window size of 7/12), as shown in Table \ref{abla:backbone}. The comparisons are conducted by changing  the backbone, whereas other settings remain unchanged. When ResNet-50 is replaced by Swin-B$^7$, we gain an improvement of 0.3\% and 0.5\% in the 5/10-shot setting (novel classes).

\section{Details of Transformer Relation Head (TRH) with Z-shot and Spatial-HOP blocks.}\label{sec:TRH}
As  Z-shot T-RH is described in Eq. \eqref{eq:zsa_q} of the main paper,
%\S \ref{z-shot-head} .
%
below we focus on describing Spatial-HOP T-RH.

This head first forms a so-called self-attention on a set $\tZ$ of support regions and $\tB$ query RoIs, respectively. We formulate its operation for $B$ query RoIs (refer \S \ref{spatial-hop head} of main paper for support regions).
 Spatial-HOP T-RH takes as input RoI features $\{\vPhi^*_b\in\mbr{2d\times N}\}_{b\in\idx{B}}$ ($2d$ because layer 5 of ResNet-50 maps $d$-dimensional features to $2d$-dimensional features) and %$\{\vPhi^*_b\in\mbr{2d\times N}\}_{b\in\idx{B}}$, and
 $\{\vpsi^*_b\in \mbr{d}\}_{b\in\idx{B}}$.
We split $\vPhi^*_{b}$ along the channel mode of dimension  $2d$ to create two new matrices $\vPhi^{*u}_{b}\in\mbr{d\times N}$ and $\vPhi^{*l}_{b}\in\mbr{d\times N}$ for $b\in\idx{B}$.
We let $\vPhi^{*l}_{b}=[\vphi^{*l}_{b,1},\cdots,\vphi^{* l}_{b,N}]\in\mbr{d\times N}$.
Self-attention is then performed over  $\tT_{b}$ containing vectors, in parallel across $B$ RoIs, \ie, $\{\tT_{b}\}_{b\in\idx{B}}$:
\begin{equation}
\tT_{b}=[\vphi^{* l}_{b,1},\cdots,\vphi^{* l}_{b,N}, \bar{\vphi}^{* u}_{b}, \mW_g\vpsi^{*}_{b}],
\label{eq:shop_b}
\end{equation}
where $\bar{\vphi}$ denotes average-pooled features (FO) and $\mW_g\in\mbr{d\times d}$ denotes a linear projection (shared between query and support representations).

Based on these representations passed through the transformer head (variables indicated by widehat $\widehat{\cdot}\;$) between support regions %(with subscript $\tZ$)
and query RoIs, we then compute relations  as follows:
\begin{align}
&\mathcal{R}^{b}_{\text{Spatial}}=
\left[ \begin{array}{c} \widehat{\vPhi}^{\dagger l}-\widehat{\vPhi}^{*l}_{b}
\end{array} \right]\in \mbr{d \times N}, \,\,\,b \in\idx{B}, \\
&\mathcal{R}^{b}_{\text{FO+HO}}=
\left[ \begin{array}{c}
\widehat{\bar{\vphi}}^{\dagger u}\cdot \widehat{\bar{\vphi}}^{*u}_{b} \\
\widehat{\vpsi}^{\dagger}\cdot \widehat{\vpsi}^{*}_{b}  \\
\end{array} \right]\in \mbr{2d},\\
&\mathcal{R}^{b}=
\left[ \begin{array}{c}
\text{Repeat}(\mathcal{R}^{b}_\text{Spatial}; N) \\
\mW^u\mathcal{R}^b_{\text{FO+HO}} \\
\end{array} \right]\in \mbr{2d \times B},
\end{align}
where the learnable weight $\mW^{(u)} \in \mbr{d \times 2d}$ projects the channel-wise concatenated matrix to $d$ dimensions, letters $l$ and $u$ indicate first and second half of channel coefficients, respectively, operator $\cdot$ indicates element-wise multiplication,  and $\text{Repeat}(\cdot; N)$ replicates spatial mode $N$ times. % times to match with the query RoIs.
The above process is shown in Fig. \ref{pipe-det}.

%Finally, the outputs of the $Z$-shot and Spatial-HOP transformer heads are individually fed into a classifier, which are aggregated to form a classification score for each query RoI. Bounding-box regressor takes the output of Spatial-HOP transformer head as input for localization.

\begin{table}[t]
%\vspace{-0.8cm}
\makeatletter\def\@captype{table}\makeatother\caption{Experimental results of different variants of Transformer Relation Head (TRH), by varying Z-shot and Spatial-HOP blocks, are in Tab. \ref{spatial_hop}. Digits $1,\cdots,4$ indicate different orders included or excluded from each experiment. ``Spatial'' is the size of spatial map (downsampled by the bilinear interpolation). Next, Tab. \ref{abla:backbone} is an ablation of different variants of Encoding Network
(5/10-shot setting on VOC2007 testing set was used in Tab. \ref{spatial_hop} and \ref{abla:backbone}).
Finally, Fig. \ref{pn} shows mAP  \wrt $\eta'$ in SigmE (10-shot protocol  on VOC2007 and COCO testing  dataset, 5-shot setting on FSOD testing  dataset).}
%\vspace{-0.25cm}
\begin{subfigure}[t]{0.45\linewidth}{
\centering
\fontsize{6}{6}\selectfont
\setlength{\tabcolsep}{3pt}
%\resizebox{\textwidth}{!}
{\begin{tabular}{c|c|c|c|c c}
\toprule
$\substack{Z-\text{shot}\\\text{\fontsize{4}{1}\selectfont (1,2,3,4)}}$&Spatial& 1 &2,3,4 & 5-shot& 10-shot \\
\midrule
\multirow{12}{*}{\checkmark}&\multirow{3}{*}{7$\times$7}&\checkmark&~&57.9&64.2\\
~&~&~&\checkmark&61.3&65.8\\
\rowcolor{LightCyan}~&~&\checkmark&\checkmark&\bf 62.3&\bf 66.9\\
\cmidrule{2-6}
&\multirow{3}{*}{5$\times$5}&\checkmark&~&58.7&63.7\\
~&~&~&\checkmark&60.3&64.3\\
~&~&\checkmark&\checkmark&61.1&65.2\\
\cmidrule{2-6}
&\multirow{3}{*}{3$\times$3}&\checkmark&~&54.8&57.9\\
~&~&~&\checkmark&56.0&59.2\\
~&~&\checkmark&\checkmark&56.6&60.1\\
\cmidrule{2-6}
&\multirow{3}{*}{1$\times$1}&\checkmark&~&45.1&49.3\\
~&~&~&\checkmark&46.8&51.9\\
~&~&\checkmark&\checkmark&47.4&52.4\\
\midrule
\multirow{12}{*}{}&\multirow{3}{*}{7$\times$7}&\checkmark&~&55.2&60.7\\
~&~&~&\checkmark&59.4&64.6\\
~&~&\checkmark&\checkmark&61.0&65.8\\
\cmidrule{2-6}
&\multirow{3}{*}{5$\times$5}&\checkmark&~&57.6&61.5\\
~&~&~&\checkmark&58.6&63.0\\
~&~&\checkmark&\checkmark&60.3&63.4\\
\cmidrule{2-6}
&\multirow{3}{*}{3$\times$3}&\checkmark&~&52.4&54.8\\
~&~&~&\checkmark&54.1&57.8\\
~&~&\checkmark&\checkmark&54.5&58.3\\
\cmidrule{2-6}
&\multirow{3}{*}{1$\times$1}&\checkmark&~&44.0&47.1\\
~&~&~&\checkmark&45.2&48.4\\
~&~&\checkmark&\checkmark&46.3&50.1\\
\bottomrule
\end{tabular}}}
\caption{\label{spatial_hop}}
\end{subfigure}
\hspace{0.30cm}
\begin{subfigure}[t]{0.45\linewidth}
\vspace{-3.5cm}
\includegraphics[trim=24 23 25 25, clip=true, width=6.cm]{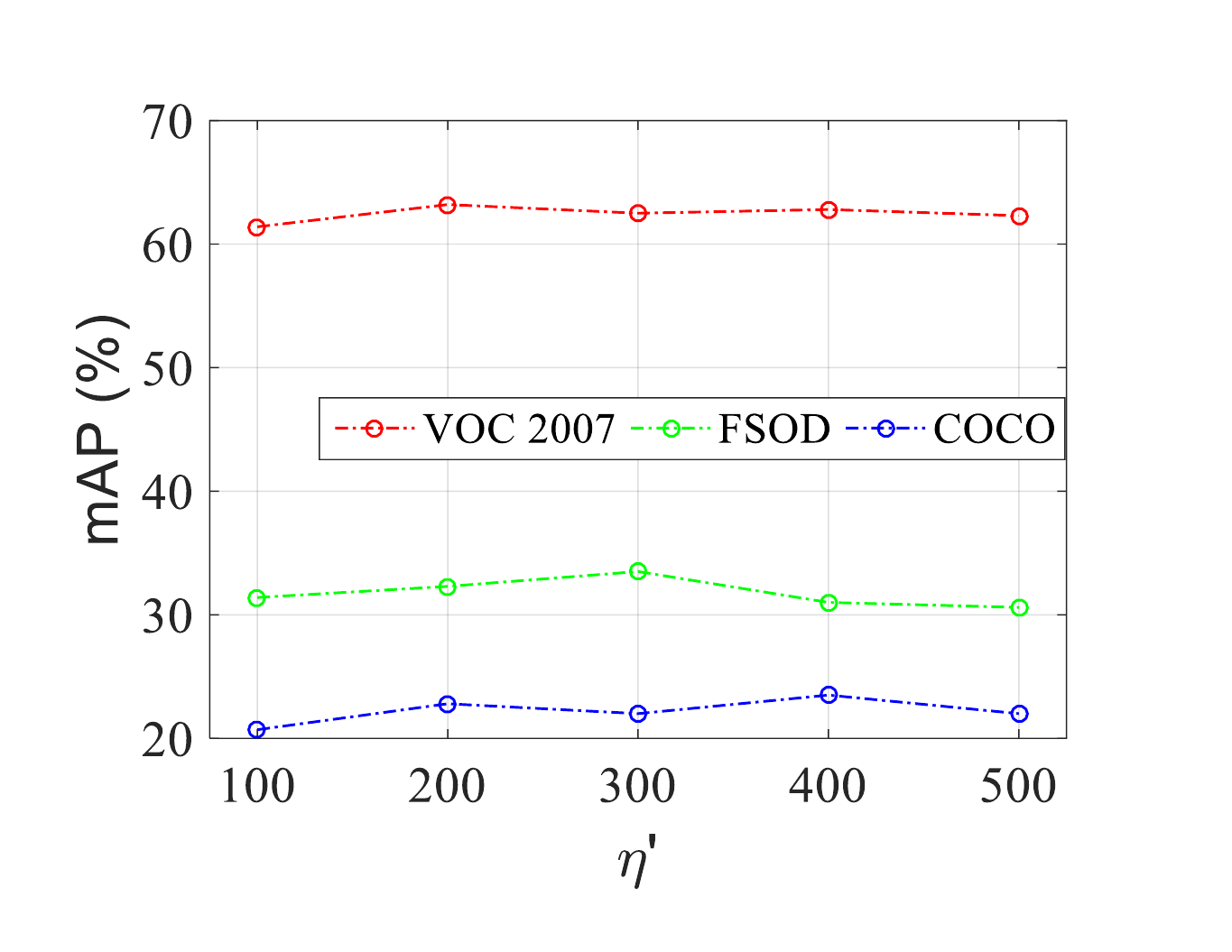}\vspace{-0.4cm}\caption{\label{pn}}
\hspace{0.5cm}
{\fontsize{6}{10}\selectfont
\setlength{\tabcolsep}{10pt}{
\begin{tabular}{c c c}
\toprule
EN&5-shot &10-shot\\
\midrule
ResNet-50&62.3&66.9\\
\rowcolor{LightCyan}Swin-B$^{_{7}}$& \bf 62.6&\bf 67.4\\
Swin-B$^{_{12}}$&62.0&66.7\\
\bottomrule
\end{tabular}}}
\caption{\label{abla:backbone}}
\end{subfigure}
\vspace{-0.8cm}
\end{table}

\section{Ablation Study on Transformer Relation Head (TRH) with
Z-shot and Spatial-HOP blocks.}\label{Spatial-HOP-TRH}
\label{sup_more}
As the supplementary setting for the top panel of Tab. \ref{STF} (in the main paper), we utilize $r\!=\!1$ in RPN and $r\!=\!2,3,4$ in TRH, achieving
2.7\%/2.4\% improvement on novel/base classes, 5-shot protocol, over the variant applied $r\!=\!1$ in both RPN and TRH.

We then conduct more ablation studies on Spatial-HOP transformer head to analyze the impact brought by each component (5/10-shot setting on novel classes, VOC 2007). The results are shown on Table \ref{spatial_hop}. Specifically, we mainly ablate three variants: spatial maps of assorted size (as in the table) with either orderless HOP representation of order $r\!=\!1$ or $r\!=\!2,3,4$, or both $r\!=\!1,2,3,4$.

Furthermore, to investigate the impact of spatial attention, we use bilinearly subsampled maps, ranging from $1 \times 1$ to $7 \times 7$ in spatial size. Not surprisingly, the Spatial-HOP head performs best when utilizing larger spatial maps, together with the orderless high-order and first-order tensor descriptors.

\begin{figure}[!htbp]
\begin{minipage}[t]{0.45\textwidth}
\centering
\includegraphics[trim=30 30 25 30, clip=true, width=6.cm]{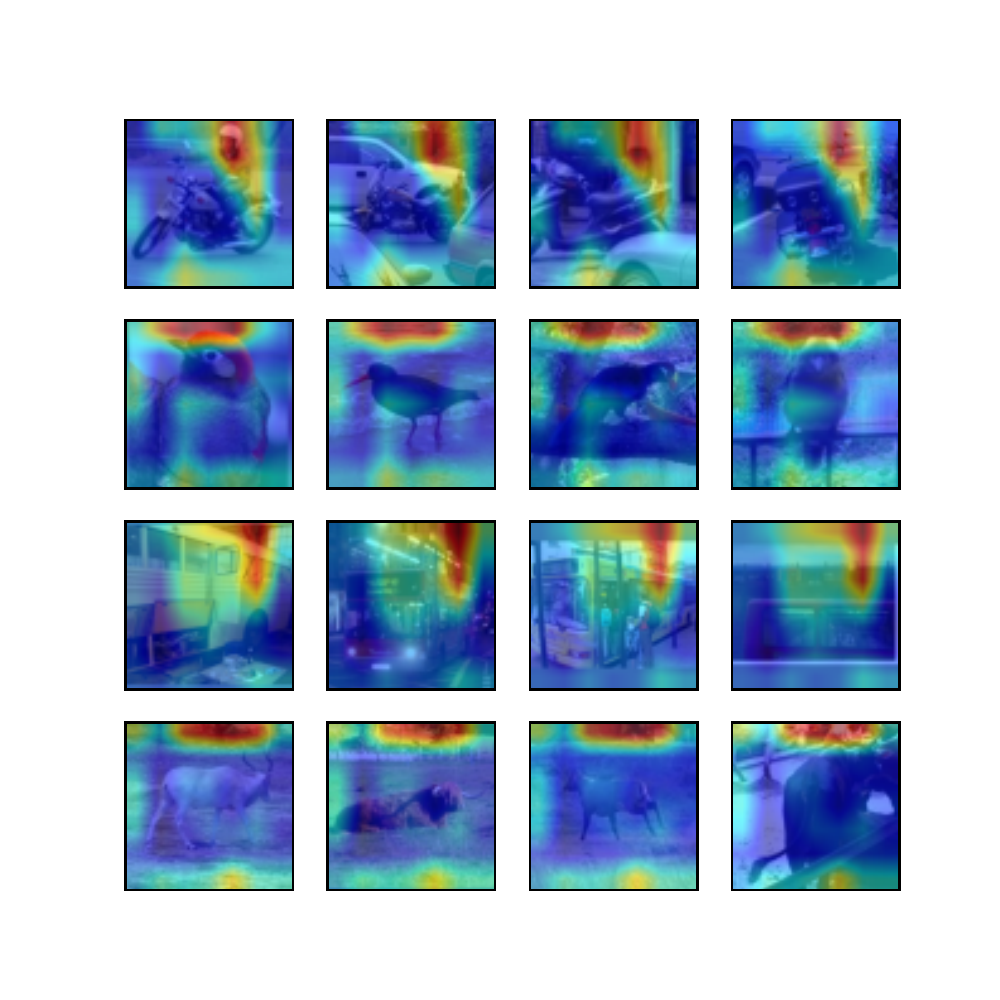}
\caption*{First-order fiber (FO) is visualised (Spatial-HOP T-RH used only spatial and FO ($r\!=\!1$) information during training)}
\end{minipage}
\hspace{0.1mm}
\begin{minipage}[t]{0.45\textwidth}
\includegraphics[trim=30 30 25 30, clip=true, width=6.cm]{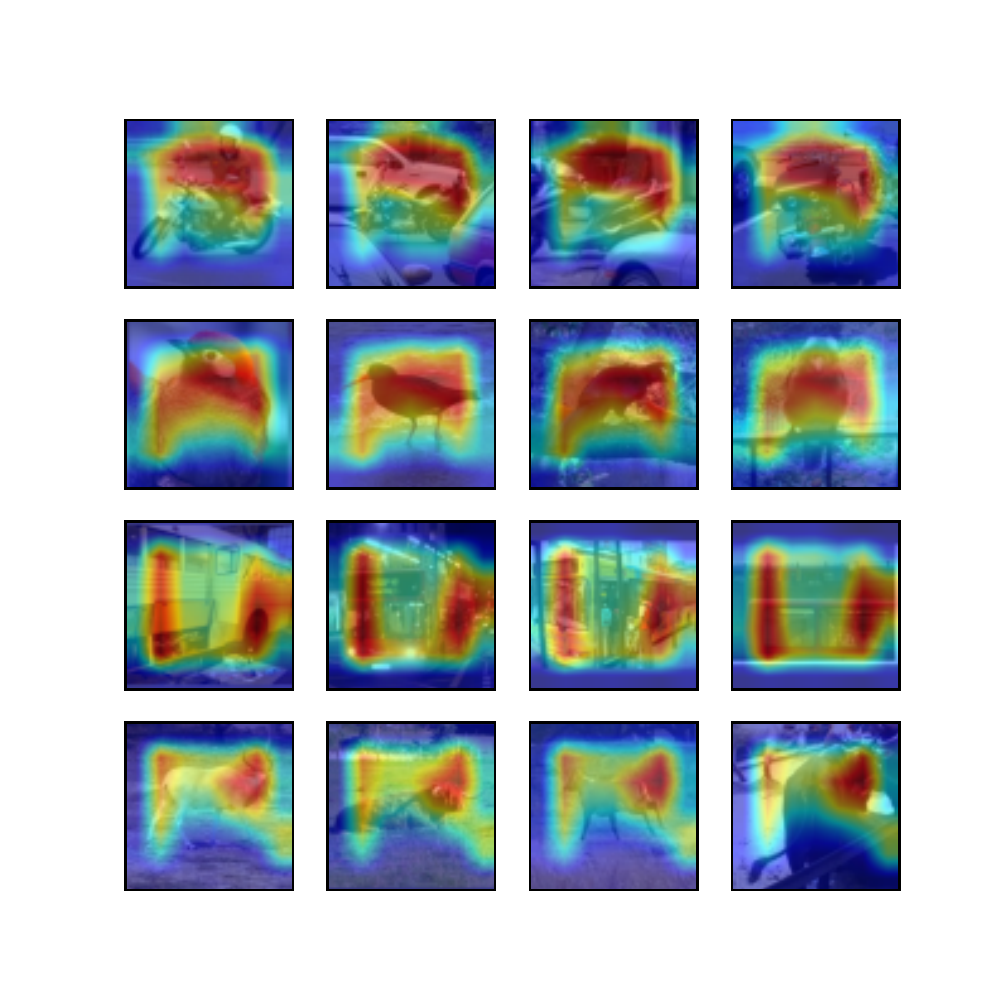}
\caption*{High-order fiber (HO) is visualised (Spatial-HOP T-RH used only spatial and HOP ($r\!=\!2,3,4$) information during training) }
\hspace{5mm}
\end{minipage}
\hspace{0.0002mm}
\begin{minipage}[t]{0.45\textwidth}
\centering
\includegraphics[trim=30 30 25 30, clip=true, width=6.cm]{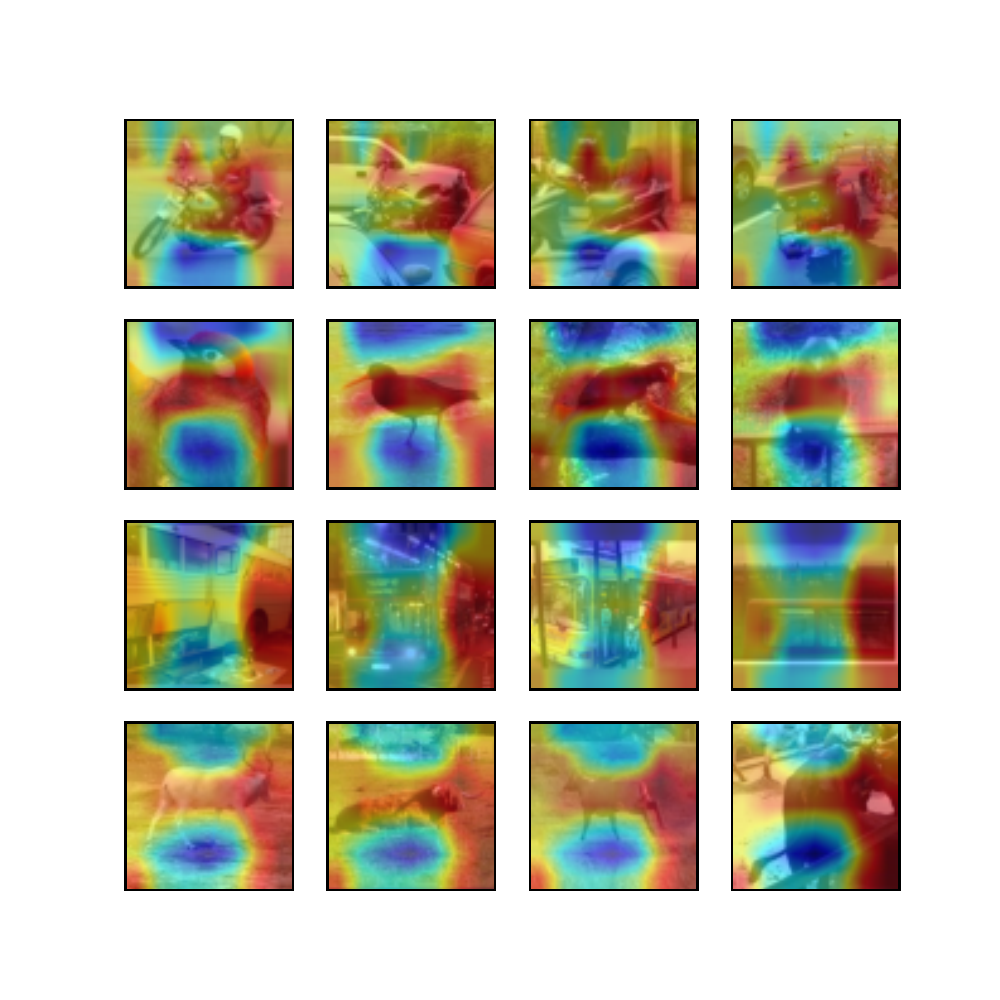}
\caption*{Spatial fibers are max-pooled and then visualised (Spatial-HOP T-RH used spatial, FO and HOP information ($r\!=\!1,2,3,4$) during training)}
\end{minipage}
\hspace{1cm}
\begin{minipage}[t]{0.45\textwidth}
\centering
\includegraphics[trim=30 30 25 30, clip=true, width=6.cm]{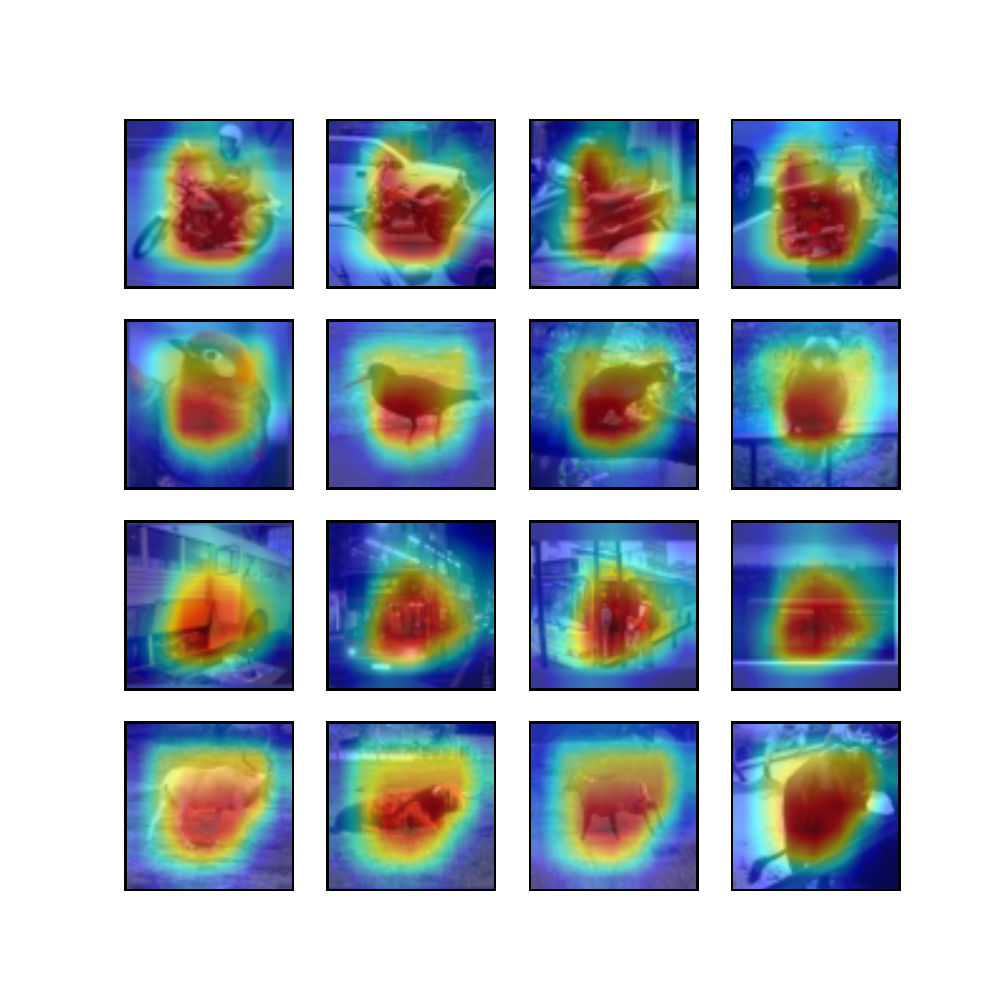}
\caption*{First-order fiber (FO) and High-order fiber (HO) are averaged and then visualised (Spatial-HOP T-RH used spatial, FO and HOP information ($r\!=\!1,2,3,4$) during training)}
\end{minipage}
%\vspace{-0.1cm}
\caption{Visualization of attention maps of self-attention \wrt support regions. The results are produced on VOC2007 test set, novel classes (motorbike, bird, bus and cow). %The Spatial-HOP transformer head equips with spatial distributions and either first-order $r\!=\!1$(top panel left) or high-order $r\!=\!2,3,4$ (top panel right) only,  both $r\!=\!1,2,3,4$ (bottom panel). Best viewed with zoom-in.
See text for detailed descriptions.
}
%\vspace{-0.3cm}
\label{vis:atten}
\end{figure}

%Firstly, we performed training where Spatial-HOP T-RH used only spatial and first-order  information (FO) during training. To obtain the picture, we picked  $\bar{\vphi}^{\dagger u}_{\tZ}$ from Eq. \eqref{eq:shop_q} and we looked how it correlates with the $N$ spatial representations $\vphi^{\dagger l}_{\tZ,1},\cdots,\vphi^{\dagger l}_{\tZ,N}$. To that end, we passed these `spatial fibers' and FO representation via the RBF kernel of Eq.  \eqref{eq:rbf}, and we then reshaped $N$ into the spatial map ($7\!\times\!7$ size).
%
%\tT_{b}\equiv\{\vphi^{* l}_{b,1},\cdots,\vphi^{* l}_{b,N};\bar{\vphi}^{* u}_{b};\mW_g\vpsi^{*}_{b}\}.
%\label{eq:shop_b}
%
%

\section{Visualization of Attention Maps of the Spatial-HOP block.}\label{visual_att}

To explain why our model benefits from the combination of spatial attention, and orderless first-order and high-order representations, we provide  qualitative results based on displaying attention maps.

Firstly, we performed training where Spatial-HOP T-RH used only spatial and first-order  information (FO) during training. To obtain the picture, we picked  $\bar{\vphi}^{\dagger u}$ from Eq. \eqref{eq:shop_q} and we looked how it correlates with the $N$ spatial representations $\vphi^{\dagger l}_{1},\cdots,\vphi^{\dagger l}_{N}$. To that end, we passed these ``spatial fibers'' and FO representation via the RBF kernel of Eq.  \eqref{eq:rbf}, and we then reshaped $N$ into the spatial map ($7\!\times\!7$ size).

Figure \ref{vis:atten} (top left) shows how the first-order representation (FO) correlates with each spatial fiber in the attention of transformer. As Spatial-HOP T-RH block uses information averaged over $K$ images of the same class in an episode ($K$-way images), each column shows one of these support images. Each row shows a different class image from $Z$-shot support images in the episode.

Subsequently, we performed training where Spatial-HOP T-RH used only spatial and high-order information (HO) during training. Thus, we picked the high-order representation $\mW^{(g})\vpsi^{\dagger}$ from Eq. \eqref{eq:shop_q} and we looked how it correlates with the $N$ spatial representations $\vphi^{\dagger l}_{1},\cdots,\vphi^{\dagger l}_{N}$. To that end, we passed these ``spatial fibers'' and HO representation via the RBF kernel of Eq.  \eqref{eq:rbf}, and we then reshaped $N$ into the spatial map ($7\!\times\!7$ size).

Figure \ref{vis:atten} (top right) shows how the high-order representation (HO) correlates with each spatial fiber in the attention of transformer. As before, we visualise $K\! \times\!Z$ images from an episode given the $K$-way  $Z$-shot problem.

Comparing FO an HO representations, HO is by far more focused on the foreground objects that correlate in the semantic sense with the object class. This explains why HO representations help our model obtain better results compared to traditional attention mechanisms that focus only on capturing spatial correlations of a region.

Figure \ref{vis:atten} (bottom left)  shows how the spatial fibers from the attention matrix that is max-pooled along columns (we of course removed FO and HO before pooling along columns). We follow the same procedure as above, however, this time the Spatial-HOP T-RH block was utilizing the spatial, FO and HO information during training. Clearly, spatial attention can focus on complex spatial patterns in contrast to the focus of FO and HO.

Figure \ref{vis:atten} (bottom right)  shows how the first-order representation (FO), averaged with the high-order representation (HO), correlate with each spatial fiber in the attention of transformer. We follow the same procedure as above, and still
 use the spatial, FO and HO information in the Spatial-HOP T-RH block during training. Clearly, utilizing $r\!=\!1,2,3,4$ compares favourably with utilizing either $r\!=\!1$ or $r\!=\!2,3,4$ during training.

\section{Impact of $\eta'$ of SigmE.}\label{sigme}
According to Section \ref{sec:approach}, TSO benefits from element-wise PN, realized by the SigmE operator in Eq. \eqref{eq:sigme1}, which depends on the parameter $\eta'$.  Figure \ref{pn} shows  that $\eta'\!=\!200$ is a good choice on VOC dataset but $\eta'\!=\!300/400$  helps  obtain the best results on FSOD/COCO dataset. Overall, our approach is not overly sensitive to this parameter, and setting $\eta'\!=\!200$ on all  datasets if a good choice.

\section{Hyperparameters on the FSOD and COCO datasets.}\label{abla:datasets}

\begin{table}[!htbp]
%\vspace{-0.1cm}
\makeatletter\def\@captype{table}\makeatother\caption{Ablation studies on the FSOD and COCO datasets (5/10-shot, novel classes), \wrt the effect of varying (a) the number
of heads used in T-Heads Attention, as shown in Tab. \ref{fsodTA}, and (b) the number of TENET blocks as shown in Tab. \ref{fsodTB}. mAP of variants of High-order Tensor Descriptors (HoTD) with TSO ($\eta_r\!>\!1$) and without TSO ($\eta_r\!=\!1$) is in Tab. \ref{tem_fsod}.}
%\vspace{-0.25cm}
\begin{subfigure}[t]{0.25\linewidth}{
\centering
\fontsize{6}{8}\selectfont
\setlength{\tabcolsep}{3pt}{
\begin{tabular}{c c c}
\toprule
\multirow{2}{*}{$TA$} & FSOD& COCO \\
\cmidrule{2-3}
&5-shot&10-shot\\
\midrule
1& 30.5&20.1\\
\rowcolor{LightCyan}2& \bf31.7&22.3\\
4& 31.2&22.6\\
\rowcolor{LightCyan}8& 30.8&\bf 23.5\\
16& 30.0&23.0\\
32& 29.4&21.8\\
64& 29.5&21.5\\
\bottomrule
\end{tabular}}}
\caption{\label{fsodTA}}
\end{subfigure}
\hspace{-0.3cm}
\begin{subfigure}[t]{0.25\linewidth}{
\centering
\fontsize{6}{10.3}\selectfont
\setlength{\tabcolsep}{3pt}{
\begin{tabular}{c c c}
\toprule
\multirow{2}{*}{$TB$} & FSOD&COCO\\
\cmidrule{2-3}
&5-shot&10-shot\\
\midrule
1& 31.7&23.5\\
\rowcolor{LightCyan}2& \bf 33.5 &24.2\\
\rowcolor{LightCyan}3& 32.6&\bf 25.1\\
4& 31.0&24.8\\
5& 31.2&23.1\\
\bottomrule
\end{tabular}}}
\caption{\label{fsodTB}}
\end{subfigure}
%\hspace{0.01cm}
\begin{subfigure}[t]{0.50\linewidth}{
\centering
\fontsize{6}{11.7}\selectfont
\setlength{\tabcolsep}{1pt}
%\resizebox{\textwidth}{!}
{\begin{tabular}{ccc|c|c|cc||c|cc}
\toprule
\multicolumn{3}{c|}{$r$}&\multirow{2}{*}{$\substack{\text{dim.}\\\text{\fontsize{4}{1}\selectfont split}}$}& \multirow{2}{*}{$\substack{\eta_r\\\text{\fontsize{4}{1}\selectfont(FSOD)}}$}&\multicolumn{2}{c||}{5-shot} & \multirow{2}{*}{$\substack{\eta_r\\\text{\fontsize{4}{1}\selectfont(COCO)}}$}&\multicolumn{2}{c}{10-shot} \\
\cmidrule{1-3}
\cmidrule{6-7}
\cmidrule{9-10}
2&3&4&~&~&$AP_{50}$&$AP_{75}$&~&$AP_{50}$&$AP_{75}$\\
\midrule
\checkmark&~&~&~&7&33.1&29.6&10&25.7&17.5\\
\checkmark&\checkmark&~& 3:1 &7,7&33.7&30.4&10,10&26.0&18.2\\
\rowcolor{LightCyan}\checkmark&\checkmark&\checkmark&5:2:1&7,7,7 &{\bf35.4}&{\bf31.6}&10,10,10&{\bf27.4}&{\bf19.6}\\
\checkmark&\checkmark&\checkmark&5:2:1&1,1,1&30.8&28.4&1,1,1&22.1&14.3\\
\bottomrule
\end{tabular}}}
\caption{\label{tem_fsod}}
\end{subfigure}
%\vspace{-0.6cm}
\end{table}

Tables \ref{fsodTA} and \ref{fsodTB} present the impact of the number of head used in T-Heads Attention ($TA$) and TENET block ($TB$) on results. We fix the $\sigma\!=\!0.5$ (the best value of standard deviation of the RBF kernel of transformers, selected by cross-validation on FSOD and COCO dataset) and then we investigate $TA$ and $TB$ (the number of attention units per block, and the number of blocks, respectively). Two heads together with two blocks are the best on the FSOD dataset, while eight heads aligned with three blocks yield the best results on the COCO dataset. Table \ref{tem_fsod}  shows results on  FSOD and COCO \wrt the dimension split along the feature channel (\eg, if $r\!=\!2,3$, ratio 3:1 means that three parts of channel dimension are taken to form the second-order representation, and one part of channel dimension is taken to form the third-order representation). The table also shows the impact of $\eta_r$ of TSO on results, where $\eta_r$ are individual parameters for each order $r$. Overall, using all three orders, as denoted by $r\!=\!2,3,4$, outperforms a second-order representation, indicated by  $r\!=\!2$. Importantly, TSO is used when $\eta_r\!>\!1$. Without TSO ($\eta_r\!=\!1$), results drop by a large margin, which highlights the practical importance of TSO on results.

\section{Comparison with QA-FewDet/DeFRCN fine-tuning/meta-testing setting (Table \ref{exp_rb_tab}).}
 Below we compare our method with  with QA-FewDet \citelatex{QA_sup}/DeFRCN \citelatex{DeFRCN_sup}.

\setlength{\tabcolsep}{0.8pt}
\begin{table}[h]
%\vspace{-1cm}
\centering
\fontsize{6}{6}\selectfont
\caption{Comparison with QA-FewDet/DeFRCN (mAP\%).}
\label{exp_rb_tab}
\begin{tabular}{l|c|ccccc|ccccc|ccccc}
\toprule
\multirow{2}{*}{Method} & \multirow{1}{*}{Encoding} & \multicolumn{5}{c|}{Novel Set 1} & \multicolumn{5}{c|}{Novel Set 2} & \multicolumn{5}{c}{Novel Set 3} \\
& Network &1   & 2     & 3    & 5    & 10   & 1     & 2     & 3    & 5    & 10   & 1     & 2     & 3    & 5    & 10   \\ \midrule
\multicolumn{17}{c}{\textbf{Meta-training the model on base classes, and meta-testing on novel classes}} \\ \midrule
QA-FewDet &  ResNet-101 & {41.0} & {33.2} & {35.3} & {47.5} & {52.0}   & {23.5} & {29.4} & {37.9} & {35.9} & {37.1}   & {33.2} & {29.4} & {37.6} & {39.8} & {41.5} \\
\rowcolor{LightCyan}TENET (Ours) & ResNet-50 & \textbf{43.7} & \textbf{42.1} & \textbf{43.9} & \textbf{48.2} & \textbf{54.5}  & \textbf{32.5} & \textbf{35.2} & \textbf{39.5} & \bf  37.8 & \textbf{38.7} & \textbf{34.1} & \textbf{37.0} & \textbf{38.9} & \textbf{42.0} & \textbf{45.1} \\
\midrule
\multicolumn{17}{c}{\textbf{Fine-tuning the model on novel classes, and testing on novel classes}} \\ \midrule
QA-FewDet  & ResNet-101 & {42.4} & {51.9} & {55.7} & {62.6} & {63.4} & {25.9} & {37.8} & {46.6} & {48.9} & {51.1} & 35.2 & {42.9} & {47.8} & {54.8} & {53.5} \\
DeFRCN & ResNet-101 & \bf 53.6& \bf 57.5 &\bf 61.5& 64.1& 60.8& 30.1 &38.1& 47.0& \bf 53.3 &47.9 &\bf 48.4& \bf 50.9&\bf  52.3 &54.9& 57.4\\
TENET (Ours) & ResNet-50 &46.7&52.3 &55.4&62.3&66.9& 40.3 & 41.2&44.7&49.3&52.1 &35.5&41.5& 46.0&54.4&54.6\\
\rowcolor{LightCyan}TENET(Ours) & ResNet-101 & 48.5 & {55.2} & {58.7} & \bf 65.8 & \textbf{69.0}    & \bf 42.6 & \bf 43.4 & \textbf{47.9} & {52.0} & \bf 54.2    & {37.9} & {43.6}  & {48.8} & \bf56.9 & \bf{57.6} \\
\bottomrule
\end{tabular}
%\vspace{-1cm}
\end{table}

\newpage

\section{Comparison with SOTA on MS COCO minival set (10/30 -shot) as shown in Table \ref{exp_rb_coco}.}
\setlength{\tabcolsep}{1pt}
\begin{table}[!htbp]
\centering
\fontsize{7}{7}\selectfont
\caption{Evaluations on the MS COCO minival set (10/30- shot). Methods that do not disclose all shot results are ignored and are replaced with `--'.}
\label{exp_rb_coco}
\begin{tabular}{c|c|cc|cc|cc}
            \toprule
            \multirow{2}{*}{Method}&\multirow{2}{*}{Venue}&\multicolumn{2}{c|}{$AP$} &\multicolumn{2}{c|}{{$AP_{50}$}}&\multicolumn{2}{c}{{$AP_{75}$}}\\
           &&10  & 30 & 10  & 30& 10  & 30  \\
            \midrule
            FSCE+SVD  & NeurIPS 2021& 12.0  &16.0 & -- & -- &10.4  &15.3 \\
            FADI  & NeurIPS 2021  & 12.2  &16.1 & -- & -- &11.9  &15.8 \\
            SRR-FSD & CVPR 2021  &11.3&14.7 &23.0&29.2 &9.8&13.5\\
            Zhang \etal & CVPR 2021  & 12.6& -- & 27.0 & --  & 10.9  & --\\
            \midrule
            QA-FewDet & ICCV 2021  & 11.6 & 16.5  & 23.9  & 31.9& 9.8& 15.5\\
            DeFRCN & ICCV 2021  & 18.5 & 22.6   & --  & --& --  & --\\
             \midrule
            QSAM& WACV 2022 & 13.0 &15.3 &24.7& 29.3 &12.1 &14.5 \\
             FCT& CVPR 2022  & 15.3 & 20.2   & --  & --& --  & --\\
            \midrule
  \rowcolor{LightCyan}TENET& Ours &\textbf{19.1} & \textbf{23.7}& \textbf{27.4} & \textbf{32.2}& \textbf{19.6}& \textbf{23.1}\\
            \bottomrule
        \end{tabular}
        \end{table}
\section{Mean $\pm$ std of mAP on PASCAL VOC 2007  (Table \ref{voc_mstd}).}
\label{voc_std}

\setlength{\tabcolsep}{3pt}
\begin{table}[!htbp]
\centering
\fontsize{7}{7}\selectfont
\caption{Evaluations
on three test splits of  VOC 2007  (mean mAP $\pm$ std).}
\label{voc_mstd}
%\hspace{-0.3cm}
\begin{tabular}{ll cccc}
\toprule
\multicolumn{2}{c}{\multirow{2}{*}{Method/Shot}}&
\multicolumn{4}{c}{Mean$\pm$std}\\
%\cline{3-18}
\cmidrule{3-6}
&\multicolumn{1}{c}{}& 1 &3 &5&\multicolumn{1}{c}{10}\\
%\hline
\midrule
FRCN &ICCV12  & 7.6$\pm$3.1&23.5$\pm$4.5&32.3$\pm$3.3&36.4$\pm$6.0\\ %& \cite{re32}
FR &ICCV19 &16.6$\pm1.9$&25.0$\pm$1.7&34.9$\pm$4.3&42.6$\pm$3.4\\%&\cite{f10ObjectDetection}
Meta& ICCV19 &14.9$\pm$3.9&30.7$\pm$3.2&40.6$\pm$4.5&48.3$\pm$2.5\\%&\cite{f12}
FSOD& CVPR20&25.4$\pm$3.2&32.0$\pm$4.8&42.2$\pm$4.2&47.9$\pm$3.9\\% &\cite{fsod}
NP-RepMet& NeurIPS20&37.6$\pm$3.4&41.6$\pm$1.5&45.4$\pm$2.8&47.8$\pm$2.1\\% & \cite{neg}

PNSD &ACCV20 &31.3$\pm$4.4&36.2$\pm$4.2&44.5$\pm$3.8&49.9$\pm$5.4\\
MPSR &ECCV20 &33.9$\pm$7.2&44.3$\pm$5.2&47.7$\pm$6.2&53.1$\pm$6.2\\%&\cite{multi}
TFA &ICML20&31.4$\pm$6.7&40.5$\pm$4.6&46.8$\pm$8.6&48.3$\pm$7.0\\
FSCE &CVPR21&31.4$\pm$9.3&44.8$\pm$4.9&51.1$\pm$7.7&55.9$\pm$5.6\\
CGDP+FRCN &CVPR21&33.1$\pm$5.6&43.7$\pm$2.3&50.0$\pm$6.0&54.8$\pm$6.6\\
TIP &CVPR21&24.0$\pm$2.6&38.4$\pm$4.0&45.2$\pm$4.3&52.5$\pm$5.3\\
FSOD\textsuperscript{up}&ICCV21&36.8$\pm$5.2&45.1$\pm$3.8&50.1$\pm$4.6&54.5$\pm$5.5\\
QSAM&WACV22&26.1$\pm$3.5&35.4$\pm$2.9&45.4$\pm$3.6&51.9$\pm$3.8\\
\midrule
\rowcolor{LightCyan}TENET&(Ours) &{\bf 40.8$\pm$3.6}&{\bf 48.7$\pm$4.7}&{\bf55.3$\pm$3.1}&\bf{57.9$\pm$5.8}\\
\bottomrule
\end{tabular}
%\vspace{-0.5cm}
\end{table}

\bibliographystylelatex{splncs04}
\bibliographylatex{egbib}

%\processdelayedfloats

\end{document}